%% file: main.tex
\documentclass{article}

\usepackage[main,final]{paper}

\usepackage[utf8]{inputenc} \usepackage[T1]{fontenc}    \usepackage{url}            \usepackage{booktabs}       \usepackage{amsfonts}       \usepackage{nicefrac}       \usepackage{microtype}      \usepackage{xcolor}
\usepackage[ruled]{algorithm2e}
\usepackage{float}        \usepackage{placeins}     \usepackage{enumitem}
\usepackage{amsmath}

\usepackage{amsmath}
\usepackage{amsthm}
\usepackage{amssymb}

\theoremstyle{plain}
\newtheorem{theorem}{Theorem}[section]

\newtheorem{lemma}[theorem]{Lemma}

\newtheorem{corollary}[theorem]{Corollary}

\newtheorem{fact}[theorem]{Fact}
\newtheorem{assumption}[theorem]{Assumption}

\theoremstyle{definition}

\theoremstyle{remark}
\newtheorem{remark}[theorem]{Remark}

\input{macros.tex}

\usepackage{graphicx}
\usepackage{tikz}
\usetikzlibrary{positioning,arrows.meta,calc}

\usepackage{enumitem}
\usepackage{multirow}

\newcommand{\restatetheorem}[2]{
\begingroup
\def\thetheorem{\ref{#1}}
\begin{theorem}[Restated]
#2
\end{theorem}
\addtocounter{theorem}{-1}
\endgroup
}

\newcommand{\restatelemma}[2]{\begingroup
\renewcommand{\thetheorem}{\ref{#1}}\begin{lemma}[Restated]
#2
\end{lemma}
\addtocounter{theorem}{-1}\endgroup
}

\newcommand{\restatecorollary}[2]{\begingroup
\renewcommand{\thetheorem}{\ref{#1}}\begin{corollary}[Restated]
#2
\end{corollary}
\addtocounter{theorem}{-1}\endgroup
}

\usepackage{yfonts}
\usepackage{mathrsfs}
\renewcommand{\cS}{\mathcal{S}}
\renewcommand{\cB}{\mathcal{B}}
\renewcommand{\cO}{\mathcal{O}}
\newcommand{\gD}{\mathfrak{D}}
\newcommand{\gA}{\mathcal{A}}
\newcommand{\dk}{k}
\renewcommand{\Pr}{\mathbb{P}}
\newcommand{\wtd}{d'}
\usepackage{soul}

\usepackage[usenames,dvipsnames]{xcolor}

\definecolor{DarkRed}{rgb}{0.368,0.097,0.078}
\definecolor{DarkBlue}{rgb}{0.2,0.2,0.6}

\usepackage[colorlinks = true,
    linkcolor = DarkBlue,
    anchorcolor = DarkBlue,
    citecolor = DarkBlue,
    filecolor = DarkBlue,
    urlcolor = DarkRed
    ]{hyperref}

\usepackage[capitalise]{cleveref}
\crefname{lemma}{lemma}{lemmas}
\Crefname{lemma}{Lemma}{Lemmas}
\crefname{corollary}{Corollary}{Corollaries}
\Crefname{corollary}{Corollary}{Corollaries}
\crefname{theorem}{Theorem}{Theorems}
\Crefname{theorem}{Theorem}{Theorems}
\crefname{proposition}{Proposition}{Propositions}
\Crefname{proposition}{Proposition}{Propositions}
\crefname{remark}{Remark}{Remarks}
\Crefname{remark}{Remark}{Remarks}
\crefname{definition}{Definition}{Definitions}
\Crefname{definition}{Definition}{Definitions}
\crefname{assumption}{Assumption}{Assumptions}
\Crefname{assumption}{Assumption}{Assumptions}
\crefname{fact}{Fact}{Facts}
\Crefname{fact}{Fact}{Facts}

\title{Capacity-Constrained Online Convex Optimization with Delayed Feedback}

\author{Alexander Ryabchenko\thanks{University of Toronto and Vector Institute; \texttt{alex.rbch.research@gmail.com}.}
  \qquad
  Idan Attias\thanks{Institute for Data, Econometrics, Algorithms, and Learning (IDEAL), hosted by the University of Illinois at Chicago and Toyota Technological Institute at Chicago; \texttt{idanattias88@gmail.com}.}
  \qquad
  Daniel M. Roy\thanks{University of Toronto and Vector Institute; \texttt{daniel.roy@utoronto.ca}.}
}

\begin{document}

\maketitle

\begin{abstract}
Online learning with delayed feedback typically assumes that the learner can track all pending rounds until their feedback arrives. In practice, tracking resources are finite, and feedback from rounds that cannot be tracked is permanently lost. In this paper, we study delayed online convex optimization (OCO) under a hard capacity constraint, where at most $C$ pending rounds can be tracked at any time. To model delay information, we introduce a semi-clairvoyant model that refines the clairvoyant assumption from prior work: rather than requiring delays to be known at prediction time, the learner observes delay expirations online, consistent with the classical unconstrained delayed setting. Our approach proceeds via a reduction to a novel ``delayed and weighted'' OCO problem, using a scheduler that randomizes which rounds are tracked and importance-weights the resulting observations. For the base problem, we propose and analyze Delayed-Weighted FTRL and its bandit analogue, obtaining regret bounds that characterize how weights interact with delayed feedback in both settings. Combining these bounds with our schedulers yields regret guarantees for capacity-constrained OCO under convex and strongly convex losses, for both first-order and bandit feedback. For first-order feedback, capacity $C = \Omega(\log T)$ suffices to recover the standard delayed OCO rates up to logarithmic factors. For bandit feedback, the standard delayed BCO rates are instead modulated by $(1 + \sigma_{\text{max}}/C)$ factors, where $\sigma_{\text{max}}$ is the maximum number of outstanding observations. This allows the regret bound to degrade gracefully when $C < \sigma_{\text{max}}$.
We complement our analysis with numerical experiments.
\end{abstract}

\section{Introduction}

Online convex optimization is a canonical framework for sequential decision-making under uncertainty \citep{hazan2023introductiononlineconvexoptimization,orabona_book}: a player repeatedly selects an action from a convex domain, incurs a convex loss, and receives either full gradient information (\emph{first-order feedback}) or only the incurred scalar loss (\emph{bandit feedback}). We refer to these two settings as Online Convex Optimization (OCO) and Bandit Convex Optimization (BCO), respectively.
A long line of work studies delayed feedback in OCO and BCO, as well as in sequential decision-making more broadly, where the observation from round~$t$ arrives only at round~$t+d_t$ \citep{joulani2013,quanrud2015,cesa16}. The vast majority of this literature, however, takes for granted that the player has unlimited capacity to track pending rounds, that is, every outstanding feedback will eventually be observed, regardless of how many are pending simultaneously.

In many practical settings, only a limited number of rounds can be tracked concurrently. A natural example arises in clinical trials: a study can actively monitor only a bounded number of patients at any given time, and outcomes are observed only for the patients selected for monitoring --- side effects may still manifest in unmonitored patients, but the trial never observes them.
\citet{capacity_constraint2025} formalized this limitation as a hard \emph{capacity constraint}, capping the number of concurrently tracked rounds by a \emph{capacity} parameter~$C$, and studied the resulting problem in the multi-armed bandit setting.

This paper extends the capacity-constrained framework from finite-action settings to convex domains, deriving regret bounds for both OCO and BCO with explicit dependence on the capacity~$C$.
The bandit setting is notoriously difficult even without delays or capacity constraints, as the player must optimize over a continuous domain from scalar loss observations alone. Beyond the move to continuous domains, we relax the clairvoyant delay model of prior work, which assumes delays are known at prediction time; our framework instead reveals delays only upon expiration, matching the information naturally available in standard delayed online learning and more closely reflecting the classical clairvoyant model in online job scheduling \citep[e.g.,][]{borodin1998online}.

\textbf{Contributions.}
Following the reduction approach of \citet{capacity_constraint2025}, we decompose capacity-constrained delayed OCO into two components: a \emph{scheduler}, which randomly selects the rounds to track subject to the capacity constraint, and a \emph{base learner}, which receives the resulting feedback stream with importance weights correcting for untracked rounds. Extending this reduction to convex domains introduces two new challenges: first, designing OCO and BCO algorithms that can accommodate feedback that is simultaneously delayed and non-uniformly weighted; and second, controlling the stochastic weight process induced by the scheduler in the presence of convex losses and bandit gradient estimation. Our contributions are as follows:

\begin{itemize}[leftmargin=0.4cm,before=\vspace{-0.2cm}]
    \item \emph{Delayed-weighted OCO and BCO} (\Cref{sec:DW-OCO}).
    We introduce OCO and BCO with feedback that is simultaneously delayed and weighted, extending the scale-free online learning framework of \citet{orabona2018scalefree} to delayed feedback. We analyze a delayed-weighted variant of Follow-the-Regularized-Leader and its bandit analogue, obtaining regret bounds that identify pairwise interactions between weights of rounds with overlapping delays as the structural cost of combining weighting with delays. These results are new for convex domains and may be of independent interest.
This framework underlies our reduction for the capacity-constrained setting.

    \item \emph{Unified proxy-delay scheduling} (\Cref{sec:reduction}). To handle delays that are unknown at prediction time, \citet{capacity_constraint2025} introduced the idea of sampling \emph{proxy delays} to estimate the true delays. We unify these constructions into a general framework parameterized by proxy-delay distributions, and introduce a new scheduler calibrated to the maximum number of outstanding observations, $\sigmax$, that yields uniform importance weights, a property crucial for our strongly convex and bandit guarantees. Importantly, the algorithm does not require prior knowledge of $\sigmax$, as it can be removed via the doubling trick (Appendix~\ref{app:doubling-tricks}).

    \item \emph{Capacity-dependent regret bounds} (\Cref{sec:results}). Combining the delayed-weighted regret bounds with the proxy-delay schedulers yields regret guarantees for capacity-constrained OCO and BCO under both convex and strongly convex losses; see Table~\ref{tab:rates}.

    For first-order feedback, capacity $C = \Omega(\log T)$ suffices to match the optimal unconstrained rates in terms of the total delay $\dtot$ and the maximum number of outstanding observations $\sigmax$, up to logarithmic factors, since $\sigmax \le \sqrt{2\dtot}$. For bandit feedback in dimension~$\dk$, the capacity constraint modulates the dimension-dependent terms by factors of order $1 + \sigmax/C$, degrading gracefully as capacity decreases. The delay-dependent terms match the optimal $\sigmax$ scaling of \citet{ryabchenko2026reductiondelayedimmediatefeedback} in the strongly convex case; for convex losses they scale as $\sqrt{T\sigmax}$ rather than $\sqrt{\dtot}$, but still improve over the $\sqrt{T\dmax}$ dependence of \citet{wan2024improvedregretbanditconvex}, since $\sigmax \le \dmax$ in general. The regime $\sigmax \ll \dmax$ is particularly natural in queueing systems: $\sigmax$ is governed by the peak arrival rate \citep[e.g.,][]{little1961}, while $\dmax$ is driven by the tail of the service-time distribution, which can produce extreme outliers even when the average service time is moderate.
    See \Cref{sec:experiments} for supporting numerical experiments.
\end{itemize}
\begin{remark}[Extensions]\label{rmk:preemption}
Our framework allows the player to \emph{preempt} tracked rounds---that is, to stop tracking them before their feedback arrives. There is also a non-preemptive variant of the problem, in which once a round is tracked, it cannot be dropped until its feedback is delivered. Several of our schedulers are non-preemptive by construction and achieve analogous guarantees in this setting; see Appendix~\ref{app:non-preemptive}.
Additionally, although our results are stated under the semi-clairvoyant and clairvoyant frameworks, they extend to the fully non-clairvoyant setting, where delays $d_t$ are revealed only at times $t+d_t$ for rounds whose feedback is observed, provided that sufficient information about delay-related quantities such as $\dtot$ or $\sigmax$ is available for tuning the algorithm parameters.
\end{remark}
\begin{table*}[t]
\caption{Notation: $T$ is the horizon, $\dk$ is the dimension, $\dtot = \sum_{t=1}^T d_t$ is the total delay, $\sigmax \le \min\{\dmax, \sqrt{2\dtot}\}$ is the maximum number of outstanding observations at any round, $C$ is the capacity. Our results are the first ones for the capacity-constrained setting (with $C = \Omega(\ln T)$); all prior work implicitly assumes unlimited capacity ($C = \infty$). }
\centering
\setlength{\tabcolsep}{10pt}
\renewcommand{\arraystretch}{1.3}

\resizebox{0.99\textwidth}{!}{\begin{tabular}{|c|c|c|c|}
\hline
\multirow{3}{*}{\textbf{Loss type}} &
\multicolumn{3}{c|}{\textbf{Regret bounds for OCO with delays}} \\ \cline{2-4}
& \multicolumn{2}{c|}{Unconstrained ($C = \infty$)} & Constrained ($C = \Omega(\ln T)$) \\ \cline{2-4}
& {\citet{quanrud2015}} &
  {\citet{qiu2025}} &
  {This work} \\
\hline
Convex & $\sqrt{\dtot} + \sqrt{T}$ & N/A & $\sqrt{\dtot}+ \sqrt{T}$ \\[2pt]
\hline
Strongly convex & N/A & $\min\{\sigmax\ln T, \sqrt{\dtot}\} + \ln T$ & $(\sigmax + 1) \ln T$ \\[2pt]
\hline
\end{tabular}}

\vspace{0.2cm}

\resizebox{0.99\textwidth}{!}{\begin{tabular}{|c|c|c|c|}
\hline
\multirow{3}{*}{\textbf{Loss type}} &
\multicolumn{3}{c|}{\textbf{Regret bounds for BCO with delays (up to log-factors)}} \\ \cline{2-4}
& \multicolumn{2}{c|}{Unconstrained ($C = \infty$)} & Constrained ($C = \Omega(\ln T)$) \\ \cline{2-4}
& {\citet{wan2024improvedregretbanditconvex}} &
  {\citet{ryabchenko2026reductiondelayedimmediatefeedback}} &
  {This work} \\
\hline
Convex & $\sqrt{T\dmax} + T^{3/4}\sqrt{k}$ & $\sqrt{\dtot} + T^{3/4}\sqrt{k}$ & $\sqrt{T\sigmax} + T^{3/4}(1+\tfrac{\sigmax}{C})^{1/4}\sqrt{k}$ \\[2pt]
\hline
Strongly convex & $\dmax + T^{2/3}k^{2/3}$ & $\sigmax + T^{2/3}k^{2/3}$ & $\sigmax + T^{2/3}(1+\tfrac{\sigmax}{C})^{1/3}k^{2/3}$ \\[2pt]
\hline
\end{tabular}}

\label{tab:rates}
\end{table*}

\textbf{Related work.}
The capacity-constrained framework was introduced by \citet{capacity_constraint2025} for experts and multi-armed bandits; this work extends it to OCO and BCO. Delayed online learning traces back to \citet{weinberger2002}, whose subsampling argument for fixed delays $d_t = d_{0}$ implies minimax rates of order $\Theta(\sqrt{T(d_0+1)})$ and $\Theta((d_0+1)\ln T)$ when applied to non-delayed problems with square-root and logarithmic regret, respectively. These fixed-delay rates imply $\Omega(\sqrt{\dtot + T})$ and $\Omega((\sigmax + 1)\ln T)$ lower bounds for arbitrary delays under convex and strongly convex losses, respectively. Matching upper bounds for first-order OCO were obtained by \citet{quanrud2015,joulani2016,flaspohler2021} for convex losses and by \citet{qiu2025} for strongly convex losses; earlier, \citet{wan2021onlinestronglyconvexoptimization} achieved the rate $O(\dmax \ln T)$.
Delayed BCO was initiated by \citet{heliou2020} and refined by \citet{bistritz2022,wan2024improvedregretbanditconvex}. A unified view via reduction to non-delayed OCO appears in \citet{ryabchenko2026reductiondelayedimmediatefeedback}, achieving optimal delay-dependent terms with respect to $\sigmax$ and $\dtot$. Beyond OCO, delayed feedback has been studied in adversarial bandits \citep{cesa16,thune2019,zimmert2020}, MDPs \citep{lancewicki2022,jin2023}, and combinatorial settings \citep{vanderhoeven2023}. Additional related work is discussed in Appendix~\ref{app:related-work}.

\section{Problem Setting and Notation}

Let $T \in \N$ denote the time horizon, and let $\cK \subset \R^{\dk}$ be a non-empty, convex, and closed domain of dimension $\dk \in \N$, equipped with a norm $\tnorm{\cdot}$ and its dual norm $\tnorm{\cdot}_{\star}$.

We study a $T$-round game between a player and an adversary with delayed, capacity-constrained feedback (Figure~\ref{fig:setting}). Before the game begins, the adversary fixes a sequence of convex, differentiable losses $f_t : \cK \to \R$ and delays $d_t \in \{0, \ldots, T-t\}$ for $t \in [T]$.\footnote{The restriction $d_t \le T - t$ is without loss of generality: feedback arriving after round $T$ cannot affect any prediction.} At each round $t$, the player selects $x_t \in \cK$ and incurs the (unobserved) loss $f_t(x_t)$. The corresponding feedback—$\nabla f_t(x_t)$ in the \emph{first-order} model or $f_t(x_t)$ in the \emph{bandit} model—arrives only at the end of round $t + d_t$, and is observed if and only if $t$ belongs at that time to a \emph{tracking set} $\cS \subseteq [T]$ of maximum size $C \in \N$ (the \emph{capacity}), maintained by the player throughout the game. At the start of each round, the player may \emph{preempt} (i.e., remove from $\cS$) any subset of round indices, irrevocably discarding their pending feedback; they may then add the current round's index to $\cS$, subject to the capacity constraint. If $C=T$, the learner can track every round and observe all delayed feedback, recovering the standard delayed-feedback setting. The goal is to achieve strong performance while maintaining a tracking set of sublinear size.

We consider two natural frameworks for capacity-constrained delayed feedback, which differ in how the delay sequence is revealed to the player. In the \emph{clairvoyant} framework, $d_t$ is revealed to the player at the start of round $t$. In the \emph{semi-clairvoyant} framework, the player learns $d_t$ only at round $t + d_t$: if $t \in \cS$, the player receives indexed feedback for round $t$, $(t, \nabla f_t(x_t))$ or $(t, f_t(x_t))$; otherwise, the player receives an empty acknowledgment $(t, \perp)$, from which $d_t$ can be inferred as the elapsed time since round $t$. Semi-clairvoyance thus preserves exactly the delay information available at runtime in standard delayed online learning~\citep{cesa16}; the capacity constraint restricts only the \emph{content} of the feedback, not the player's knowledge of when feedback would have arrived.

\begin{figure}[H]
\centering
\resizebox{0.94\columnwidth}{!}{\fbox{\parbox{0.95\columnwidth}{
    \textbf{Capacity-Constrained Online Convex Optimization with Delayed Feedback}
    \vspace{2pt}

    $\bullet$ \textit{Latent parameters:} number of rounds $T$.\\
    $\bullet$ \textit{Pre-game:} adversary fixes losses $f_{t}:\cK \to \R$ and delays $d_t \in \{0, \ldots, T-t\}$ for $t \in [T]$.

    \vspace{3pt}
    Player initializes empty tracking set $\cS$ of maximum size $C$.

    \noindent For each round $t=1, 2, \ldots, T$:
    \begin{enumerate}[leftmargin=0.5cm,noitemsep,before=\vspace{-0.1cm},after=\vspace{-0.1cm}]
        \item [0.] If the framework is \ul{clairvoyant}, then the player observes delay $d_t$.
        \item \textbf{Predict:} The player predicts $x_t \in \cK$, incurring loss $f_t(x_t)$.

        \item \textbf{Update tracking set (preempt \& add):}
        \begin{itemize}[noitemsep,leftmargin=0.3cm,before=\vspace{-0.05cm},after=\vspace{-0.05cm}]
            \item The player may \ul{preempt} (remove from $\cS$) any subset of indices, possibly none.
            \item The player may add round index $t$ to the tracking set $\cS$, provided $|\cS| < C$.
        \end{itemize}

        \item \textbf{Receive feedback:}\\
        For all $s$ such that $s+d_s = t$:
        \begin{itemize}[noitemsep,leftmargin=0.3cm,before=\vspace{-0.05cm},after=\vspace{-0.1cm}]
            \item If $s \in \cS$, the player observes feedback from round $s$: $(s,\nabla f_s(x_s))$ in the first-order feedback model, or $(s, f_s(x_s))$ in the bandit feedback model, \ul{and removes $s$ from $\cS$}.
            \item If $s \notin \cS$ and the framework is \ul{semi-clairvoyant}, then the player observes $(s, \perp)$, signaling that the delay for round $s$ has expired, i.e., $d_s = t-s$.
        \end{itemize}

    \end{enumerate}
}}}
\caption{Capacity-Constrained OCO with delays under first-order or bandit feedback.}
\label{fig:setting}
\end{figure}
The \emph{regret} against a comparator $u \in \cK$ is defined as $R_T(u) = \tsum_{t=1}^T (f_t(x_t) - f_t(u))$, and its expected regret  as $\widebar{R}_T(u) = \E[R_T(u)]$. Letting $x^* \in \argmin_{x\in\cK} \tsum_{t=1}^T f_t(x)$ denote the best action in hindsight, the objective is to minimize $R_T(x^*)$ almost surely or its expectation $\widebar{R}_T(x^*)$.

\textbf{Regularity assumptions.} We impose the following standard regularity assumptions, which are necessary to obtain sublinear regret guarantees in online convex optimization.
\begin{assumption}\label{assump:finite-diameter}
    The diameter of $\cK$ is bounded by $D$, i.e., $\sup_{x,y\in\cK} \norm{x-y} \le D$.
\end{assumption}
\begin{assumption}\label{assump:gradient-norm-bound}
    For all $t \in [T]$, $f_t$ has its gradient norm bounded by $G$, i.e., $\sup_{x\in\cK} \norm{\nabla f_t(x)}_{\star}\!\le\!G$.
\end{assumption}

Assumptions \ref{assump:finite-diameter} and \ref{assump:gradient-norm-bound} hold throughout the paper. When deriving stronger guarantees, we also impose the following strong convexity condition, stated explicitly when invoked.
\begin{assumption}\label{assump:strong-convexity}
For all $t \in [T]$, $f_t$ is $\lambda$-strongly convex for $\lambda > 0$, i.e., $f_t(y) \ge f_t(x) + \pair{\nabla f_t(x),\, y - x} + \tfrac{\lambda}{2}\norm{y - x}^2$ for all $x,y\in \cK$.
\end{assumption}

For BCO, we take $\tnorm{\cdot}$ to be the Euclidean norm $\tnorm{\cdot}_2$ and impose two additional assumptions.
The first-order guarantees in Theorems~\ref{thm:DW-OCO} and~\ref{thm:OCO-main-result} also
use the Euclidean norm.
\begin{assumption}\label{assump:max-val}
    For all $t \in [T]$, absolute values of $f_t$ are bounded by $M$, i.e., $\sup_{x\in\cK} |f_t(x)|\!\le\!M$.
\end{assumption}
\begin{assumption}\label{assump:balls}
    The domain satisfies $r\Ball \subseteq \cK \subseteq R\Ball$ for some radii $0 < r < R$.
\end{assumption}

\textbf{Notation.} For $n \in \N$, we let $[n] = \{1,\ldots,n\}$. For any sequence $(\mathrm{a}_t)_{t=1}^T$, indexed over $[T]$, we write $\mathrm{a}_{\textnormal{tot}} = \tsum_{t=1}^T \mathrm{a}_t$ and $\mathrm{a}_{\textnormal{max}} = \max_{t\in [T]} \mathrm{a}_t$. We denote by $\Ball$ and $\Sphere$ the unit Euclidean ball and sphere centered at the origin in $\R^{\dk}$, respectively. For any $S \subseteq \R^{\dk}$ and $c \in \R$, we let $cS = \{cx : x \in S\}$.

\subsection{Notation for Delays and the Tracking Set}
\label{subsec:prelim-delays}

The interplay between delays and the capacity constraint requires keeping careful track of which past rounds have produced feedback, which are still pending, and which the player is currently tracking. We introduce this notation and present a few standard identities in delayed online learning.

\textbf{Observed and backlog sets.} At the start of each round $t \in [T]$, the past rounds split into those whose feedback has already arrived and those whose feedback is still pending. We denote these by
\begin{align*}
    \cO_t \;=\; \{s \in [t-1] : s + d_s < t\},
    \qquad
    \cB_t \;=\; \{s \in [t-1] : t \le s + d_s\},
\end{align*}
respectively; together they partition the past, $\cO_t \cup \cB_t = [t-1]$ and $\cO_t \cap \cB_t = \emptyset$.

We write $\sigma_t = |\cB_t|$ for the backlog set size (or just \emph{backlog}) at round $t$, and define
\[
    \textstyle\sigma_{\textnormal{max}} \;=\; \max_{t \in [T]} \sigma_t,
    \qquad
    d_{\textnormal{max}} \;=\; \max_{t \in [T]} d_t,
    \qquad
    d_{\textnormal{tot}} \;=\; \tsum_{t=1}^T d_t,
\]
the maximum backlog, maximum delay, and total delay over the game. The next lemma collects two identities relating these quantities, which we will use repeatedly to bound the cost of the backlog.

\begin{lemma}[Backlog properties]\label{lem:backlog-properties}
For any delays $(d_t)_{t=1}^T$ with $d_t \in \{0, \ldots, T-t\}$, the following hold:
\begin{align*}
    \tsum_{t=1}^T \sigma_t = d_{\textnormal{tot}} \qquad \text{and} \qquad \sigma_{\textnormal{max}} \le \min\{\sqrt{2 d_{\textnormal{tot}}},\, d_{\textnormal{max}}\}.
\end{align*}
\end{lemma}
The first identity is the standard average-case quantity governing regret rates in delayed online learning~\citep{cesa16}; the second controls the worst-case backlog.

\textbf{Tracking set evolution.} Among the pending rounds, the player tracks only a subset. To reason about which rounds are tracked at each point in time, we consider the post-preemption snapshot
\begin{align*}
    \cS_t \;=\; \cbr{s \in [T] : \text{$s$ lies in $\cS$ immediately after preemption in round $t$}}.
\end{align*}
It satisfies $\cS_t \subseteq \cB_t$, $|\cS_t| \le C$, and $\cS_{t+1} \subseteq \cS_t \cup \{t\}$.
The first inclusion holds because every round in $\cS$ has been played but has not yet delivered feedback and round $t$ itself has not yet been added at this point. The second is the capacity constraint. The third reflects that between round $t$ and round $t+1$, the set may only gain index $t$ (if added) and may only lose indices (through feedback or preemption).

\subsection{Notation for Bandit Feedback and Smoothing}

The bandit setting requires constructing gradient estimates from function values alone. We follow the standard single-point approach of \citet{flaxman2004}, which replaces each loss function with a smoothed surrogate whose gradient admits an unbiased estimator from a single bandit query.

\textbf{Smoothing.} Under Assumption~\ref{assump:balls}, the inclusion $r\Ball \subseteq \cK$ implies $(1 - \delta/r)\,\cK + \delta\,\Ball \subseteq \cK$ for all $\delta \in (0, r]$, and so any point of the shrunken domain $(1-\delta/r)\cK$ admits a $\delta$-ball of perturbations that stays in $\cK$. For such $\delta$ and any integrable $f : \cK \to \R$, we define the \emph{$\delta$-smoothing} of $f$ as
\[
    f^{\delta} : (1 - \delta/r)\cK \to \R,
    \qquad
    f^{\delta}(x) \;=\; \E_{v \sim \Unif(\Ball)}\bigl[f(x + \delta v)\bigr].
\]
The key properties of $f^{\delta}$ are summarized in the following theorem.

\begin{theorem}[\citet{flaxman2004}]\label{thm:SPGE}
    For any $\delta\in(0,r]$ and continuous $f:\cK\to\R$:
    \begin{enumerate}[noitemsep,before=\vspace{-0.3cm},leftmargin=0.5cm,after=\vspace{-0.2cm}]
        \item The $\delta$-smoothing $f^{\delta}$ is differentiable with gradients $\nabla f^{\delta}(x) = \tfrac{\dk}{\delta}\E_{u\sim \Unif(\Sphere)}[f(x+\delta u)\,u]$. \item If $f$ is convex ($\lambda$-strongly convex), so is $f^{\delta}$.
        \item If $f$ is $G$-Lipschitz, then $|f^{\delta}(x) - f(x)| \le G\delta$ and $\tnorm{\nabla f^{\delta}(x)} \le G$ for $x \in (1-\delta/r)\cK$.
    \end{enumerate}
\end{theorem}

The first item underpins single-point gradient estimation; the remaining two ensure that $f^{\delta}$ preserves the regularity of $f$ and stays uniformly close to it.

\section{Online Convex Optimization with Delayed and Weighted Feedback}\label{sec:DW-OCO}

This section introduces a weighted variant of OCO with delayed feedback (Figure~\ref{fig:weighted-setting}) that serves as the foundation for the reductions in subsequent sections. The game itself involves no tracking set management; the connection to capacity-constrained OCO emerges only through the reductions. Beyond its role in our reductions, this game gives, to our knowledge, the first treatment of \emph{scale-free} online convex optimization~\citep{orabona2018scalefree} under delayed feedback.

In this game, the adversary fixes losses $f_t : \cK \to \R$, delays $d_t \in \{0, \ldots, T-t\}$, and weights $w_t \ge 0$ in advance. At each round $t$, the player predicts $x_t \in \cK$ and incurs the weighted loss $w_t f_t(x_t)$. The feedback for round $t$ is revealed at the end of round $t+d_t$, together with the weight $w_t$.

\begin{figure}[H]
\centering
\fbox{\parbox{0.95\columnwidth}{
    \textbf{Online Convex Optimization with Delayed and Weighted Feedback}
    \vspace{2pt}

    $\bullet$ \textit{Latent parameters:} number of rounds $T$.\\
    $\bullet$ \textit{Pre-game:} adversary selects $f_{t}:\cK \to \R$, $d_t \in \{0, \ldots, T-t\}$, and $w_t \in [0, \infty)$ for $t \in [T]$.

    \vspace{3pt}
    \noindent For each round $t=1, 2, \ldots, T$:
    \begin{enumerate}[leftmargin=0.5cm,noitemsep,before=\vspace{-0.1cm},after=\vspace{-0.1cm}]
        \item The player predicts $x_t \in \cK$, incurring loss $w_t f_t(x_t)$.
        \item For every $s$ such that $s+d_s = t$, the environment reveals $(s, w_s, \nabla f_s(x_s))$ in the first-order feedback model, or $(s, w_s, f_s(x_s))$ in the bandit feedback model.
    \end{enumerate}
}}
\caption{OCO with weights and delays under first-order or bandit feedback.}
\label{fig:weighted-setting}
\end{figure}

We measure performance through the \emph{weighted regret} against a comparator $u \in \cK$,
\begin{align*}
    R^{w}_T(u) \;=\; \tsum_{t=1}^T w_t \bigl(f_t(x_t) - f_t(u)\bigr),
\end{align*}
and its expectation $\widebar{R}_T^w(u) = \E[R_T^w(u)]$.
We highlight that this problem is \emph{not} equivalent to delayed OCO on the rescaled losses $f_t'(x) = w_t f_t(x)$: since the weights $w_t$ may vary widely, the vacuous rescaling reduction yields bounds scaling with $\max_t w_t$, whereas we seek finer weight dependence --- a property central to the reductions developed later.

We now produce algorithms for this game in both first-order and bandit feedback models, following the delayed-update paradigm of \citet{flaspohler2021}: at each round, the learner applies an update based on all feedback items that arrive in that round, with the modification that each item is now scaled by its weight $w_s$. Section~\ref{subsec:DW-OCO} develops this for first-order feedback, yielding a weighted variant of Follow-The-Regularized-Leader (FTRL). Section~\ref{subsec:DW-BCO} extends the construction to bandit feedback using the single-point estimator of \citet{flaxman2004} and prediction drift techniques from \citet{ryabchenko2026reductiondelayedimmediatefeedback}.

\subsection{First-Order Feedback: DW-FTRL}
\label{subsec:DW-OCO}

Given a non-increasing learning-rate sequence $(\eta_t)_{t=1}^T \subset (0,\infty)$, define the \emph{regularization increments} $\alpha_1 = 1/\eta_1$ and $\alpha_t = 1/\eta_t - 1/\eta_{t-1}$. The \emph{delayed-weighted FTRL} (DW-FTRL) updates are given by
\begin{align}
    \makebox[5em][c]{\text{DW-FTRL:}}\quad x_{t+1}
        &= \argmin_{x\in \cK}\,\pair{x, \tsum_{s \in \cO_{t+1}} w_s \nabla f_s(x_s)} +  \tsum_{s=1}^t \frac{\alpha_s}{2}\,\tnorm{x_s\!-\!x}^2. \label{eq:DW-FTRL}
\end{align}
Each update aggregates the weighted gradients revealed so far and regularizes against the running sequence of iterates. The following theorem bounds the weighted regret of DW-FTRL.

\begin{theorem}\label{thm:DW-OCO}
    For all $u \in \cK$, predictions $(x_t)_{t=1}^T$ generated by DW-FTRL~\eqref{eq:DW-FTRL} satisfy
    \begin{align*}
        R^{w}_T(u) \le \tfrac{D^2}{2\eta_T} + G^2\,\tsum_{t=1}^T \eta_{t-1} (w_t^2 + w_t\tsum_{s\in \cB_t} w_s),
    \end{align*}
    where $\eta_0 = \eta_1$. Moreover, under $\lambda$-strong convexity (\ref{assump:strong-convexity}),
    \begin{align*}
        R^{w}_T(u) \le \tsum_{t=1}^T \tfrac{\alpha_t - w_t\lambda}{2}\,\tnorm{x_t-u}^2 + G^2\, \tsum_{t=1}^T \eta_{t-1} (w_t^2 + w_t \tsum_{s\in \cB_t} w_s).
    \end{align*}
\end{theorem}

In the delayed FTRL bounds of Theorem~\ref{thm:DW-OCO}, weighting manifests through pairwise interactions $w_t w_s$ for every pair of rounds $t, s$ such that $s \in \cB_t$, i.e., the feedback of $s$ is still pending at $t$. Without delays the backlog is empty and only the diagonal terms $w_t^2$ remain, recovering the standard weighted FTRL bounds from scale-free online learning~\citep{orabona2018scalefree}. Conversely, setting all weights to $w_t = 1$ recovers standard delayed FTRL bounds~\citep[see, e.g.,][]{qiu2025}.

\subsection{Bandit Feedback: DW-FTBL}
\label{subsec:DW-BCO}

We extend the delayed-weighted updates to bandit feedback through the classical \emph{center-perturbation construction} of \citet{flaxman2004}, applied to a non-increasing sequence of smoothing parameters $(\delta_t)_{t=1}^T \subset (0, r]$. Each round we sample an independent exploration direction $u_t \sim \Unif(\Sphere)$ and pick some $\sigma(u_1, \ldots, u_{t-1})$-measurable center $y_t \in \cK$ to play the perturbed point
\begin{align*}
    x_t \;=\; (1 - \delta_t/r)\,y_t + \delta_t\, u_t,
\end{align*}
which lies in $\cK$ for $\delta_t \le r$ due to Assumption~\ref{assump:balls}. Letting $\cF_t = \sigma(u_1, \ldots, u_{t-1})$, the sequence of centers $y_t$ is $\cF_t$-measurable while the predictions $x_t$ are $\cF_{t+1}$-measurable. By Theorem~\ref{thm:SPGE}, the bandit feedback $f_t(x_t)$ yields the single-point estimator
\begin{align*}
    \htg_t \;=\; \tfrac{\dk}{\delta_t}\, f_t(x_t)\, u_t,
    \qquad
    \E[\htg_t \mid \cF_t] \;=\; \nabla f_t^{\delta_t}\bigl((1 - \delta_t/r)\,y_t\bigr),
\end{align*}
unbiased for the gradient of the smoothed loss $f_t^{\delta_t}$ at the shrunken center $(1-\delta_t/r)\,y_t$.

It remains to specify the $\cF_t$-measurable centers $(y_t)_{t=1}^T$. Given a non-increasing learning-rate sequence $(\eta_t)_{t=1}^T$ with regularization increments $(\alpha_t)_{t=1}^T$ as in \Cref{subsec:DW-OCO}, we define the \emph{delayed-weighted Follow-The-Bandit-Leader} (DW-FTBL) updates by
\begin{align}
    \makebox[5em][c]{\text{DW-FTBL:}}\quad y_{t+1}
        &= \argmin_{y\in \cK}\,\pair{y,\, \tsum_{s \in \cO_{t+1}} w_s\,\htg_s} + \tsum_{s=1}^t \tfrac{\alpha_s}{2}\,\tnorm{y_s - y}^2. \label{eq:DW-FTBL}
\end{align}
This is the bandit analog of DW-FTRL~\eqref{eq:DW-FTRL}, with the weighted estimators $w_s\,\htg_s$ in place of the weighted gradients $w_s\,\nabla f_s(x_s)$. Since each $\htg_s$ is $\cF_{s+1}$-measurable, the resulting centers $(y_t)_{t=1}^T$ are indeed $\cF_t$-adapted. The next theorem bounds the expected regret of both updates.

\begin{theorem}\label{thm:DW-BCO}
    For all $u \in \cK$, predictions $(x_t)_{t=1}^T$ generated by DW-FTBL~\eqref{eq:DW-FTBL} satisfy
    \begin{align*}
        \widebar R^w_T(u) &\le \tfrac{D^2}{2\eta_T}
        + 2G^2\tsum_{t=1}^T\!\eta_{t-1}\!\rbr{\tfrac{k^2 M^2 w_t^2}{G^2 \delta_t^2 } + \tfrac{kM w_t}{G \delta_t} \sqrt{\tsum_{s\in \cB_t} w_s^2} + w_t \tsum_{s\in \cB_t} w_s} + \tfrac{3GD}{r}\tsum_{t=1}^T w_t\delta_t,
    \end{align*}
    where $\eta_0 = \eta_1$. Moreover, under $\lambda$-strong convexity (\ref{assump:strong-convexity}),
    \begin{align*}
        \widebar R^w_T(u) &\le \E\!\sbr{\tsum_{t=1}^T \tfrac{\alpha_t - w_t \lambda}{2}  \tnorm{y_t - u}^2} \\
        &+ 2G^2\tsum_{t=1}^T \eta_{t-1}\!\rbr{\tfrac{k^2 M^2 w_t^2}{G^2 \delta_t^2 } + \tfrac{kM w_t}{G \delta_t} \sqrt{\tsum_{s\in \cB_t} w_s^2} + w_t \tsum_{s\in \cB_t} w_s} + \tfrac{10 G^2}{\lambda r}\tsum_{t=1}^T w_t\delta_t.
    \end{align*}
\end{theorem}
Compared to Theorem~\ref{thm:DW-OCO}, weighting has a more severe effect under bandit feedback. Beyond the OCO-style interaction $w_t \sum_{s\in \cB_t} w_s$, the bound carries an additional cross term $\tfrac{\dk M w_t}{G \delta_t} (\tsum_{s\in \cB_t} w_s^2)^{1/2}$, through which the variance of pending gradient estimators accumulates in $\ell^2$ rather than $\ell^1$. For uniform weights $w_t \equiv w$, AM-GM absorbs this term into its neighbors, recovering the delayed bandit rate of \citet{ryabchenko2026reductiondelayedimmediatefeedback} for $w = 1$. With non-uniform weights, the $(\tsum_{s\in \cB_t} w_s^2)^{1/2}$ factor introduces an $\ell^2$ dependence on the pending weights that has no counterpart in the first-order setting.

\section{Reduction via Proxy Delay Scheduling}\label{sec:reduction}

We now reduce capacity-constrained OCO with delayed feedback (Figure~\ref{fig:setting}) to delayed-weighted OCO from the previous section (Figure~\ref{fig:weighted-setting}). For concreteness, we present this reduction under the semi-clairvoyant framework (i.e., the delay is not known at prediction time), but it also applies to the clairvoyant and non-clairvoyant frameworks; see Remark~\ref{rmk:applicability}.
The reduction is implemented by a wrapper $\cW(\gA, \gD)$ (Algorithm~\ref{alg:scheduler-proxy-delays}) that takes as input an algorithm $\gA$ for delayed and weighted OCO and is parameterized by a sequence of \emph{proxy delay distributions} $\gD = (\gD_t)_{t=1}^T$, which govern how the wrapper maintains the tracking set.

\textbf{Proxy Delay Scheduling.} The distributions $\gD$ are used to sample proxy delays $\wtd_t \sim \gD_t$ on $\Z_{\ge -1} \cup \{\infty\}$ at each round $t$, determining round $t$'s lifetime in the tracking set: round $t$ is added to $\cS$ only if $\wtd_t \ge 0$ and $\cS$ is not full, and is preempted from $\cS$ at round $t + \wtd_t + 1$ if its feedback has not arrived by then. Feedback from round $t$ is therefore observed if and only if the proxy delay covers the true delay (i.e., $\wtd_t \ge d_t$) and the tracking set was not full when round $t$ was considered for admission (i.e., $|\cS_t| < C$). The choice of $\gD$ navigates a fundamental tradeoff: heavier-tailed distributions are more likely to cover true delays but keep rounds in $\cS$ longer, risking saturation.

\textbf{Algorithm Overview.} The wrapper treats a base algorithm $\gA$ for OCO with delayed and weighted feedback as a black box and feeds it a simulated feedback stream constructed from the observations available in the original capacity-constrained problem. Using $\gD$, the weights are set as importance-weighted indicators of whether the corresponding feedback is observed.

\begin{algorithm}[H]
\caption{Semi-Clairvoyant Proxy-Delay Importance-Weighted Wrapper $\cW(\gA, \gD)$}\label{alg:scheduler-proxy-delays}
\SetKwInOut{Access}{Access}
\Access{ Algorithm $\gA$ for OCO with delayed and weighted feedback (\Cref{fig:weighted-setting}).}
\SetKwInOut{Inputs}{Inputs}
\Inputs{ Proxy delay distributions $\gD = (\gD_t)_{t=1}^T \subset \Delta(\Z_{\ge -1} \cup \{\infty\})$.}
\vspace{0.1cm}
Initialize tracking set $\cS = \emptyset$ of maximum size $C$.

\textbf{For} round $t = 1, 2, \ldots, T$:
\begin{enumerate}[leftmargin=0.5cm,noitemsep,before=\vspace{0cm},after=\vspace{-0.1cm}]
    \item \textbf{Predict:} Start round $t$ for $\gA$ and output its prediction as $x_t$.

    \item \textbf{Update tracking set (preempt \& add):}
    \begin{itemize}[leftmargin=0.5cm,before=\vspace{-0.00cm},noitemsep]
        \item \ul{Preempt expired tracking:} For all $\tilde{s} \in \cS$ such that $\tilde{s}+\wtd_{\tilde{s}} < t$, remove $\tilde{s}$ from $\cS$.
        \item \ul{Add current round:} Sample proxy delay $\wtd_t \sim \gD_t$. Add $t$ to $\cS$ if $|\cS| < C$ and $\wtd_t \ge 0$.
    \end{itemize}

    \item \textbf{Receive feedback:} \\
    \textbf{For} all $s$ such that $s+d_s=t$:
    \begin{itemize}[leftmargin=0.5cm,before=\vspace{-0.0cm},noitemsep]
        \item \textbf{If} $s \in \cS$ (i.e., timeout sufficient: $\wtd_s \ge d_s$): \\
        Receive $(s, v_s)$; set $w_s = 1/\Pr(\wtd_s \ge d_s)$; forward $(s, w_s, v_s)$ to $\gA$; remove $s$ from $\cS$.

        \item \textbf{If} $s \notin \cS$ (i.e., timeout expired early: $\wtd_s < d_s$, or skipped due to capacity $C$): \\
        Receive $(s, \perp)$; set $w_s = 0$; forward $(s, 0, 0)$ to $\gA$.
    \end{itemize}
\end{enumerate}
\end{algorithm}

Every round $t$, the wrapper outputs $\gA$'s prediction $x_t$ as its own. At round $s + d_s$, if $s$ is still in $\cS$ the wrapper forwards $(s, w_s, v_s)$ to $\gA$ with importance weight $w_s = 1/\Pr(\wtd_s \ge d_s)$ and observation $v_s \in \{\nabla f_s(x_s),\, f_s(x_s)\}$; otherwise it forwards the zero placeholder $(s, 0, 0)$. The following theorem formalizes the properties of the random weights and bounds the wrapper's regret by the weighted regret of $\gA$ plus a penalty for rounds in which the tracking set $\cS$ is full.

\begin{theorem}[Regret of Algorithm~\ref{alg:scheduler-proxy-delays}]\label{thm:wrapper-regret}
Suppose the proxy delays are sampled independently and
$\Pr(\wtd_t \ge d_t)>0$ for every $t\in[T]$.
The weights forwarded to $\gA$ satisfy
\[
    w_t = \tfrac{\ind(|\cS_t| < C)\,\ind(\wtd_t \ge d_t)}{\Pr(\wtd_t \ge d_t)}, \quad \E[w_t] \le 1,\quad \E[w_t^2] \le \tfrac{1}{\Pr(\wtd_t \ge d_t)}, \quad \E[w_t w_s] \le 1,
\]
for all distinct $t,s \in [T]$. Moreover, the expected regret satisfies
\[
    \widebar R_T(u) \;\le\; \E\!\sbr{\tsum_{t=1}^T w_t\bigl(f_t(x_t) - f_t(u)\bigr)} + GD\, \tsum_{t=1}^T \Pr\bigl(|\cS_t| = C\bigr).
\]
\end{theorem}

Theorem~\ref{thm:wrapper-regret} reduces the analysis of capacity-constrained OCO with delayed
feedback to bounding two quantities: the weighted regret of $\gA$
(Theorem~\ref{thm:DW-OCO} or~\ref{thm:DW-BCO}) and the saturation probabilities $\Pr(|\cS_t|=C)$.
Under the oblivious-adversary assumption, the tracking sets and weights
are determined entirely by the fixed delay sequence and the scheduler's
random draws. Conditioning on these draws fixes the weights while
preserving the base learner's randomization. Theorems~\ref{thm:DW-OCO} and~\ref{thm:DW-BCO}
therefore apply conditionally, and averaging over scheduler randomness
yields bounds on the expected weighted regret. The choice of proxy-delay
distributions $\gD_t$ controls the tradeoff between the
weighted-regret bound and the saturation probabilities. The next section
presents concrete choices of proxy-delay distributions.

\begin{remark}[Applicability across frameworks]\label{rmk:applicability}
Beyond the semi-clairvoyant framework, this reduction immediately extends to the clairvoyant framework, since the player can simulate the semi-clairvoyant empty acknowledgments from known delays. It also extends to the non-clairvoyant framework provided that $\gA$ ignores zero-weight observations, i.e., its predictions are unchanged when a zero-weight item is forwarded; the wrapper can then simply skip the forwarding step for untracked rounds. The algorithms of Section~\ref{sec:DW-OCO} satisfy this property, since a zero weight zeroes out the corresponding gradient term in the update.
\end{remark}

\subsection{Concrete Proxy distributions}

We now give concrete choices of proxy distributions $\gD$. Lemma~\ref{lem:proxy-condition} states a generic sufficient condition under which the tracking set saturates with probability at most $\exp(-C)$, and the two corollaries instantiate this condition for specific families of $\gD_t$. The first scheduler, an adaptation of the Pareto-based construction of~\citet{capacity_constraint2025}, sets the per-round observation probability $p_t$ inversely proportional to the delay $d_t$. The second is new to this work: it sets a uniform per-round observation probability $p$ inversely proportional to the maximum backlog $\sigmax = \max_t |\cB_t|$ (\Cref{subsec:prelim-delays}), trading adaptivity for a constant weight scale $1/p$ across rounds, which is essential for the bounds we obtain in the strongly convex and bandit settings. This scheduler requires knowledge of~$\sigmax$; however, the resulting regret bounds in \Cref{sec:results} can nevertheless be extended to unknown~$\sigmax$ via a novel doubling trick in Appendix~\ref{app:doubling-tricks}.

\begin{lemma}\label{lem:proxy-condition}
Suppose the proxy delay distributions $\gD = (\gD_t)_{t=1}^T$ from Algorithm~\ref{alg:scheduler-proxy-delays} satisfy
\[
    \tsum_{s \in \cB_t \cup \{t\}} \Pr(\wtd_s \ge t-s) \;\le\; \tfrac{C}{8} \qquad \text{for every } t \in [T].
\]
Then $\Pr\bigl(|\cS_t| = C\bigr) \le \exp(-C)$ for every $t \in [T]$.
\end{lemma}
The bound follows immediately from a Chernoff inequality on the size of $\cS_t$.

\begin{corollary}[Pareto scheduler]\label{cor:pareto-scheduler}
For each $t \in [T]$, let $H_t = \tsum_{j=1}^t \tfrac{1}{j}$. Take $\gD_t$ to be the shifted floored Pareto distribution \footnote{A $\Pareto(x_m, \alpha)$ random variable has tail CDF $\Pr(X \ge x) = \min\{1,\, (x_m/x)^{\alpha}\}$ for $x > 0$.}
\[
\gD_t \;=\; \floor{\Pareto\bigl(\tfrac{C}{16H_t},\,1\bigr)} - 1.
\]
Then $\Pr(\wtd_t \ge d_t) = \min\bigl\{1,\, \tfrac{C}{16 H_t (d_t+1)}\bigr\}$ and $\Pr(|\cS_t| = C) \le \exp(-C)$ for every $t \in [T]$.
\end{corollary}

\begin{corollary}[Bernoulli scheduler]\label{cor:bernoulli-sigmax}
Suppose $\sigmax$ is known. Let $p = \min\bigl\{1,\, \tfrac{C}{8(\sigmax+1)}\bigr\}$, and for every $t \in [T]$ take $\gD_t$ to be the two-point distribution on $\{-1, \infty\}$ with
\[
    \textstyle \Pr_{\wtd \sim \gD_t}(\wtd = \infty) \;=\; p,
    \qquad
    \Pr_{\wtd \sim \gD_t}(\wtd = -1) \;=\; 1 - p.
\]
Then $\Pr(\wtd_t \ge d_t) = p$ and $\Pr(|\cS_t| = C) \le \exp(-C)$ for every $t \in [T]$.
\end{corollary}

The two corollaries sit at opposite ends of a tradeoff. Corollary~\ref{cor:pareto-scheduler} provides $p_t = \tilde O(\tfrac{C}{d_t+1})$, adapting to each round's individual delay, while Corollary~\ref{cor:bernoulli-sigmax} achieves the uniform $p = O(\tfrac{C}{\sigmax + 1})$. Each is useful in different regimes, as the next section illustrates.

\section{Regret Bounds for Capacity-Constrained OCO with delayed feedback}\label{sec:results}

Combining the weighted regret bounds of \Cref{sec:DW-OCO} with the proxy-delay wrapper of \Cref{sec:reduction} yields regret guarantees for capacity-constrained OCO and BCO under delayed feedback. The two theorems below cover the convex and strongly convex cases under first-order and bandit feedback, respectively, and constitute the main results of this work.
The mild restriction $C \ge \ln T + 1$ ensures that the saturation probabilities of Lemma~\ref{lem:proxy-condition} sum to a constant over the horizon. The learning rates $\eta_t$ and smoothing parameters $\delta_t$ depend on problem quantities that may not be known in advance: either $\dtot$ and $T$, or $\sigmax$. The former dependence can be removed by the two-dimensional doubling trick of~\citet{bistritz2022}, and the latter by a new doubling trick developed in Appendix~\ref{app:doubling-tricks}.

\begin{theorem}[Capacity-Constrained OCO with Delays]\label{thm:OCO-main-result}
    Suppose $C \ge \ln T + 1$, and consider $\cW(\gA, \gD)$ (\Cref{alg:scheduler-proxy-delays}), where $\gA$ is DW-FTRL~\eqref{eq:DW-FTRL}. With learning rates $\eta_t = \tfrac{D/G}{\sqrt{\dtot + T}}$ and $\gD$ from Corollary~\ref{cor:pareto-scheduler}, $\cW(\gA, \gD)$ guarantees
    \begin{align*}
        \widebar R_T(x^*) = O\rbr{GD\,{\sqrt{\dtot + T}}}.
    \end{align*}
    Under $\lambda$-strong convexity (\ref{assump:strong-convexity}), for rates $\eta_t = \frac{1}{t\lambda}$ and $\gD$ from Corollary~\ref{cor:bernoulli-sigmax}, $\cW(\gA, \gD)$ guarantees
    \begin{align*}
        \widebar R_T(x^*) = O\rbr{\tfrac{G^2}{\lambda} (\sigmax + 1) \ln T}.
    \end{align*}
\end{theorem}

\paragraph{Small capacities.}
The logarithmic-capacity requirement arises from controlling the
saturation penalty $GDTe^{-C}$ and is not claimed to be necessary.
For a known fixed delay $d$, the block-sampling construction of
\citet[Section~3]{capacity_constraint2025} can be adapted to OCO as
follows. Partition the horizon into blocks of $d+1$ rounds, play one
action throughout each block, and track one uniformly sampled round
per block. The sampled gradient estimates the average block-loss
gradient without bias and arrives by the end of the following block.
Two OCO learners on alternating blocks therefore achieve expected
regret $O(GD\sqrt{T(d+1)})$ with capacity $C=2$.
Whether constant capacity suffices to recover the total-delay
guarantee for arbitrary delay sequences remains open.

To state the BCO bounds, we introduce the dimensionless ratio $\nu = \frac{M}{Gr}$, where $M$ is the absolute value bound (Assumption~\ref{assump:max-val}).

\begin{theorem}[Capacity-Constrained BCO with Delays]\label{thm:BCO-main-result}
    Suppose $C \ge \ln T + 1$, and consider $\cW(\gA, \gD)$, where $\gA$ is DW-FTBL~\eqref{eq:DW-FTBL} and $\gD$ is from Corollary~\ref{cor:bernoulli-sigmax}. With smoothing parameters $\delta_t = r\,\min\cbr{1,\, \tfrac{\sqrt{\nu k}}{t^{1/4}}\, (1 + \tfrac{\sigmax}{C})^{1/4}}$ and learning rates $\eta_t = \tfrac{D/G}{ t^{1/2}\sqrt{\sigmax + (\frac{\nu k}{\delta_t/r})^2 (1+\frac{\sigmax}{C})}}$, $\cW(\gA, \gD)$ guarantees
    \begin{align*}
        \widebar R_T(x^*) = O\!\rbr{GD\rbr{\sqrt{T\sigmax} + T^{3/4}\rbr{1 + \tfrac{\sigmax}{C}}^{1/4}\!\sqrt{\nu k}}}.
    \end{align*}
    Under $\lambda$-strong convexity (\ref{assump:strong-convexity}), with $\eta_t = \tfrac{1}{t\lambda}$ and $\delta_t = r\,\min\cbr{1,\, (\tfrac{\nu^2 k^2 \ln(t+1)}{t})^{1/3} ({1+ \tfrac{\sigmax}{C}})^{1/3}}$,
    \begin{align*}
        \widebar R_T(x^*) = O\!\Big(\tfrac{G^2}{\lambda} \Big(\sigmax\ln T + (T^2 \ln T)^{1/3} (1 + \tfrac{\sigmax}{C})^{1/3} (\nu k)^{2/3}\Big)\Big).
    \end{align*}
\end{theorem}

Under the uniform-weight scheduler, importance weighting contributes
a factor $\E[w_t^2]=O\bigl(1+\frac{\sigmax}{C}\bigr)$ to the gradient
estimator's second-moment bound. Balancing this contribution against
smoothing error yields the capacity exponents $\frac{1}{4}$ and $\frac{1}{3}$ for
convex and strongly convex losses, respectively.

Both theorems follow from \Cref{thm:wrapper-regret} by combining the weighted regret bounds of \Cref{thm:DW-OCO,thm:DW-BCO} with the saturation bounds of Corollaries~\ref{cor:pareto-scheduler} and \ref{cor:bernoulli-sigmax}; full proofs are deferred to Appendix~\ref{app:proofs-oco-bco}.

\section{Experiments}\label{sec:experiments}

\newcommand{\capacityfigurewidth}{0.80\linewidth}
\newcommand{\experimentcodefootnote}{\footnote{Our code is publicly available at \mbox{\url{https://github.com/alex-rbch/cc-oco}}.}}
\input{experiments/section.tex}

\section{Discussion and Future Work}

This work gives the first regret bounds for capacity-constrained OCO, in both first-order and bandit settings under convex and strongly convex losses. For first-order OCO, logarithmic capacity suffices to match unconstrained delayed rates; for BCO, capacity enters regret through $(1 + \sigmax/C)$ factors.

Several directions remain open. First, our BCO bounds for convex losses scale as $\sqrt{T\sigmax}$ in the delay-dependent term rather than the optimal $\sqrt{\dtot}$ of the unconstrained setting; closing this gap or establishing a separation is an open problem. Second, the capacity-constrained framework bears a structural resemblance to finite-buffer queueing systems, where jobs that find the buffer full are dropped. Our schedulers exploit this connection, but the precise relationship between queueing theory and online learning under capacity constraints warrants deeper exploration. Third, extending from oblivious to \emph{adaptive} adversaries, who select losses and delays based on the player's history, appears to require fundamentally new tools and is a compelling direction for future work.

\begin{ack}
Idan Attias is supported by the National Science Foundation under Grant ECCS-2217023, through the Institute for Data, Econometrics, Algorithms, and Learning (IDEAL). Daniel M. Roy is supported by the funding through NSERC Discovery Grant and Canada CIFAR AI Chair at the Vector Institute.
\end{ack}

\bibliographystyle{plainnat}
\bibliography{refs}

\clearpage
\appendix

\section{Additional Related Work}\label{app:related-work}

\paragraph{Comparison to \citet{capacity_constraint2025}.}
The capacity-constrained framework was recently introduced by \citet{capacity_constraint2025} for finite-action settings, specifically prediction with expert advice and multi-armed bandits. To manage the tracking set under the capacity limit, they developed a modular approach that decouples the observation scheduling policy from the base learning algorithm, proposing techniques such as batch partitioning and Pareto-distributed proxy delays. While we adopt their high-level separation of schedulers and base learners, extending this framework to continuous convex domains (OCO and BCO) introduces fundamentally new theoretical challenges. In particular, the proxy-delay scheduler of \citet{capacity_constraint2025} inherently induces non-uniform observation probabilities; while discrete-action algorithms can absorb these via time-varying loss scales, such variable importance weights are highly detrimental to continuous bandit gradient estimation and strongly convex objectives. To address this, we introduce the delayed-weighted OCO and BCO framework, and design a novel unified proxy-delay scheduler calibrated to the maximum number of outstanding observations~$\sigmax$ that yields the uniform importance weights necessary for our continuous-domain bounds. Furthermore, their preemptive scheduler requires prior knowledge of the maximum delay~$\dmax$ in the non-clairvoyant setting. In contrast, our approach accommodates a semi-clairvoyant model and removes the need for prior knowledge of delay bounds via the doubling trick.

\paragraph{Connections to Online Scheduling.}
Our model is closely related to online job and interval scheduling
\citep{lipton1994,woeginger1994,borodin1998online,pinedo2022scheduling}.
In online interval scheduling, jobs arrive sequentially with release times and processing intervals, and the algorithm must assign jobs to a bounded number of machines while respecting overlap constraints.
Under the correspondence with our setting,
rounds become jobs,
tracking capacity corresponds to the number of available machines,
and delays determine the durations for which machines remain occupied.
Tracking a round occupies one unit of capacity until feedback arrives or the round is preempted.
The quantity $\sigmax$ corresponds to the peak concurrency level of overlapping jobs, whereas $\dmax$ reflects the longest processing duration and may therefore be much larger.
This perspective helps explain regimes in which
$\sigmax \ll \dmax$,
which naturally arise in queueing systems with moderate average load but heavy-tailed service times.
The distinction between preemptive and non-preemptive schedulers in our framework also parallels classical scheduling models.

\paragraph{Connections to Streaming Algorithms.}

Our framework also bears conceptual similarities to streaming algorithms
\citep{alon1996space,muthukrishnan2005data}.
Streaming algorithms process massive data streams under strict memory constraints, typically maintaining compact randomized summaries or sketches rather than storing the entire stream.
Analogously, in our setting the learner cannot retain all pending observations and must selectively track only a bounded subset of rounds.
Both settings therefore rely on randomized sampling and compressed representations to preserve global performance guarantees despite limited resources.
Classic examples include the AMS sketch
\citep{alon1996space},
Count-Min Sketch
\citep{cormode2005improved},
and HyperLogLog
\citep{flajolet2007hyperloglog}.

\paragraph{Online convex optimization.}
Online convex optimization (OCO) and its feedback models have been studied extensively. Early work by \citet{zinkevich2003} introduced online convex programming and established $O(\sqrt{T})$ regret bounds for online gradient methods. For first-order feedback, subsequent research established the standard $\Theta(\sqrt{T})$ minimax scaling for general convex losses \citep[see, e.g.,][]{abernethy2008}. For $\lambda$-strongly convex losses, logarithmic regret is achievable; in particular, \citet{hazan2007} provided $O(\ln T)$ guarantees, which match the $\Omega(\ln T)$ lower bound of \citet{abernethy2008}. Comprehensive background and further developments can be found in standard texts \citep{hazan2023introductiononlineconvexoptimization, orabona_book}.

\paragraph{Bandit convex optimization.}
Bandit convex optimization (BCO) was initiated by \citet{flaxman2004}, who introduced the classical single-point gradient estimator and obtained a regret of order $O(T^{3/4}\dk^{1/2})$. Subsequently, \citet{agarwal2010} developed optimal algorithms for multi-point bandit feedback and sharpened the theoretical guarantees. For two-point feedback, they achieved $\widetilde{O}(T^{1/2}\dk^2)$ regret for convex losses and $O(\dk^2\ln T)$ regret under strong convexity. For single-point feedback, they also demonstrated improved rates under additional curvature, including a $T^{2/3}$-type guarantee for strongly convex losses (specifically, $O((T^2 \ln T)^{1/3} \dk^{2/3})$ up to problem-dependent constants). Complementary results for stochastic convex optimization under bandit feedback were established by \citet{agarwal2011stochasticconvexoptimizationbandit}. For two-point feedback, \citet{shamir2015} later introduced an optimal algorithm with an improved dimension dependence of $O(T^{1/2}\dk^{1/2})$. A parallel line of work designs (nearly) minimax-optimal algorithms for general BCO that approach the $\Omega(\dk\sqrt{T})$ lower bound of \citet{shamir2013} \citep[see, e.g.,][]{bubeck2015multiscaleexplorationconvexfunctions, hazan2016optimalalgorithmbanditconvex, lattimore2020improvedregretzerothorderadversarial}. Because these methods typically incur prohibitive computational costs, exhibiting exponential or high-degree polynomial dependence on the dimension~$\dk$ and/or the horizon~$T$, we refer the reader to \citet{lattimore_BCO_book} for a detailed treatment and further references.

\paragraph{Online convex optimization with delayed feedback.}
Early investigations into online learning with delayed feedback date back to \citet{weinberger2002}. A systematic treatment of delayed-feedback online learning and reduction-style approaches was developed by \citet{joulani2013}, while \citet{joulani2016} introduced a black-box reduction that processes feedback strictly in the order of its arrival. Beyond generic black-box techniques, a substantial body of literature focuses on algorithms explicitly tailored to delayed OCO \citep[e.g.,][]{langford2009slowlearnersfast, mcmahan2014, quanrud2015, flaspohler2021, wan2021onlinestronglyconvexoptimization, wu2024,wan2024nonstationary,qiu2025}. Delayed BCO has similarly received dedicated attention, beginning with delayed extensions of bandit gradient methods \citep{heliou2020} and advancing toward delay-adaptive bounds with improved delay dependence \citep{bistritz2022, wan2024improvedregretbanditconvex}. Recently, \citet{esposito2026parameterfree} reduced delayed feedback to online learning with time-varying movement costs in the dynamic regret setting.

\paragraph{Delays in related sequential decision problems.}
Delayed observations have been studied extensively in multi-armed bandits and related models \citep[see, e.g.,][]{cesa16, thune2019, bistritz2022, esposito2023, masoudian2024, masoudian2024best}. Such delays emerge naturally under resource limitations, including the capacity constraints we study \citep{capacity_constraint2025}. Delayed feedback has additionally been explored in adversarial contextual bandits with general function approximation \citep{levy2025regretboundsadversarialcontextual}, as well as in reinforcement learning and Markov decision processes (MDPs) operating with delayed or asynchronous feedback \citep[e.g.,][]{bouteiller2020reinforcement, lancewicki2022, jin2023, vanderhoeven2023}.

\section{Technical Preliminaries}

Here, we collect auxiliary facts and standard results used throughout the proofs.

\begin{fact}\label{fact:sc-diameter}
    Under Assumptions~\ref{assump:gradient-norm-bound} and \ref{assump:strong-convexity}, the diameter of $\cK$ is bounded by $\frac{2G}{\lambda}$, i.e., $\sup_{x,y\in\cK} \norm{x-y} \le \frac{2 G}{\lambda}$.
\end{fact}
\begin{proof}
    Consider $f = f_t$ for arbitrary $t\in[T]$. From the $\lambda$-strong convexity of $f$ (e.g., see \citet{orabona_book}), it holds that for all $x,y \in \cK$, $\lambda \norm{x - y}^2 \le \pair{\nabla f(x) - \nabla f(y),\, x - y} \le 2G \norm{x - y}$. Hence, $\sup_{x,y\in\cK} \norm{x-y} \le \frac{2 G}{\lambda} < \infty$.
\end{proof}

\begin{fact}[\citet{capacity_constraint2025}, Fact F.2]\label{fact:harmonic}
    Let $t \in \N$. For the sequence of harmonic numbers $(H_s)_{s=1}^t = (\tsum_{s'=1}^s 1/s')_{s=1}^t$, it holds that $\sum_{s=1}^{t} \frac{1}{2H_s (t-s+1)} \le 1$.
\end{fact}

\begin{proof}
    For every $t\in \N$,
    \begin{align*}
        \tsum_{s=1}^{t} \tfrac{1}{2H_s (t-s+1)}
        &= \tsum_{s=1}^{\tceil{t/2}} \tfrac{1}{2H_s (t-s+1)} + \tsum_{s=\tceil{t/2}+1}^{t} \tfrac{1}{2H_s (t-s+1)}\\
        &\le \tsum_{s=1}^{\tceil{t/2}} \tfrac{1}{2 \tceil{t/2}} + \tsum_{s=\tceil{t/2}+1}^{t} \tfrac{1}{2H_{\tceil{t/2}} (t-s+1)}\\
        &\le \tfrac{\tceil{t/2}}{2\tceil{t/2}} + \tfrac{H_{\tceil{t/2}}}{2H_{\tceil{t/2}}}\\
        &= 1\qedhere.
    \end{align*}
\end{proof}

The following results are standard in the analysis of FTRL.

\begin{theorem}[\citet{orabona_book}, adapted from Lemmas 7.1 and 7.5]\label{thm:FTRL-orabona}
    Let $\cK \subset \R^{\dk}$ be a convex compact set, let $(\eta_n)_{n=0}^N \subset (0,\infty)$ be a non-increasing sequence, and let $\psi_0,\ldots,\psi_N : \R^{\dk} \to \R$ be differentiable functions such that each $\psi_n$ is $(1/\eta_n)$-strongly convex with respect to $\tnorm{\cdot}$ on $\cK$. Assume $z_1 \in \argmin_{z \in \cK} \psi_0(z)$ and $\psi_0(z_1)=0$. Given loss vectors $(c_n)_{n=1}^N \subset \R^{\dk}$, define $z_{n+1} = \argmin_{z \in \cK} \cbr{\psi_n(z) + \tsum_{m=1}^n \pair{c_m,z}}$ for $n=1,\ldots,N$. Then the minimizers $z_{n+1}$ are unique. Moreover, for all $u \in \cK$,
    \begin{align*}
       \tsum_{n=1}^N \pair{c_n,z_n-u}
       \le
       \psi_N(u)
       +
       \tsum_{n=1}^N
       \rbr{
          \tfrac{1}{2}\eta_{n-1}\tnorm{c_n}_{\star}^2
          +
          \psi_{n-1}(z_{n+1})-\psi_n(z_{n+1})
       } .
    \end{align*}
\end{theorem}

\begin{lemma}[\citet{qiu2025}, Lemma A.2 (modified)]\label{lem:stability-lemma}
    Let $\cK\subset \R^{\dk}$ be a convex compact set. For $i \in \{0, 1\}$, let function $\phi_i:\R^{\dk}\to\R$ be differentiable and $\lambda_i$-strongly convex with respect to $\tnorm{.}$ on $\cK$. Consider $z_i \in \argmin_{x\in\cK} \pair{w_i,\, x}+ \phi_i(x)$. It holds that
    \begin{align*}
        \tfrac{\lambda_0 + \lambda_1}{2} \norm{z_0 - z_1}^2 \le \pair{w_0 - w_1,\, z_1 - z_0} + (\phi_1(z_0) - \phi_0(z_0)) - (\phi_1(z_1) - \phi_0(z_1)).
    \end{align*}
\end{lemma}
\begin{proof}
    For $i \in \{0,1\}$, let $h_i(x) = \pair{w_i,\, x} + \phi_i(x)$ denote the $\lambda_i$-strongly convex function minimized at point $z_i$.
    Therefore, by the strong convexity of $h_i$, it holds that
    \begin{align*}
        h_i(z_{1-i}) - h_i(z_i) \ge \tfrac{\lambda_i}{2}\norm{z_{1-i} - z_{i}}^2.
    \end{align*}
    By summing the above inequality for both $i\in\{0,1\}$, we conclude
    \begin{align*}
        \pair{w_0 - w_1,\, z_1 - z_0} + (\phi_1(z_0) - \phi_0(z_0)) &- (\phi_1(z_1) - \phi_0(z_1))\\
        &= h_0(z_{1}) - h_0(z_0) + h_1(z_{0}) - h_1(z_1)\\
        &\ge \tfrac{\lambda_0 + \lambda_1}{2} \norm{z_0 - z_1}^2. \qedhere
    \end{align*}
\end{proof}

\section{OCO with Delayed and Weighted Feedback: Proofs}

We prove the regret bounds for DW-FTRL (\Cref{thm:DW-OCO}) and DW-FTBL (\Cref{thm:DW-BCO}) stated in \Cref{sec:DW-OCO}. Both proofs follow the delayed-update template: we compare the actual iterates against a fictitious sequence that has access to all gradients without delay, and bound the drift between them using the stability lemma (Lemma~\ref{lem:stability-lemma}). In the bandit case, we control the variance of pending gradient estimators using a technique of \citet{ryabchenko2026reductiondelayedimmediatefeedback} (Lemma~\ref{lem:BCO-meandering}).
\restatetheorem{thm:DW-OCO}{
    For all $u \in \cK$, predictions $(x_t)_{t=1}^T$ generated by DW-FTRL~\eqref{eq:DW-FTRL} satisfy
    \begin{align*}
        R^{w}_T(u) \le \tfrac{D^2}{2\eta_T} + G^2\,\tsum_{t=1}^T \eta_{t-1} (w_t^2 + w_t\tsum_{s\in \cB_t} w_s),
    \end{align*}
    where $\eta_0 = \eta_1$. Moreover, under $\lambda$-strong convexity (\ref{assump:strong-convexity}),
    \begin{align*}
        R^{w}_T(u) \le \tsum_{t=1}^T \tfrac{\alpha_t - w_t\lambda}{2}\,\tnorm{x_t-u}^2 + G^2\, \tsum_{t=1}^T \eta_{t-1} (w_t^2 + w_t \tsum_{s\in \cB_t} w_s).
    \end{align*}
}
\begin{proof}
    Consider a sequence $(\wtx_t)_{t=1}^T \subset \cK$, defined by $\wtx_1 = x_1$ and
    \begin{align*}
        \wtx_{t+1} &= \argmin_{x\in \cK}\,\,  \pair{x, \tsum_{s \in [t]} w_s\nabla f_s(x_s)} + \tsum_{s=1}^t \frac{\alpha_s}{2}\,\tnorm{x_s\!-\!x}^2.
    \end{align*}
    We treat the general convex case as 0-strongly convex. For $\lambda \ge 0$, the $\lambda$-strong convexity of $f_t$ gives
    \begin{align*}
        \tsum_{t=1}^T w_t(f_t(x_t) - f_t(u))
        &\le \tsum_{t=1}^T \pair{w_t\nabla f_t(x_t), x_t-u} - \tfrac{\lambda}{2} \tsum_{t=1}^T w_t\tnorm{x_t-u}^2\\
        &\le \tsum_{t=1}^T \pair{w_t \nabla f_t(x_t), \wtx_t-u} - \tfrac{\lambda}{2} \tsum_{t=1}^T w_t\tnorm{x_t-u}^2\\
        &+ G \tsum_{t=1}^T w_t\tnorm{x_t-\wtx_t},
    \end{align*}
    where the final inequality follows from Assumption~\ref{assump:gradient-norm-bound} and the triangle inequality.

    To bound the first term, apply Theorem~\ref{thm:FTRL-orabona} to $(\wtx_t)_{t=1}^T$ with $\psi_t(x) = \tsum_{s=1}^t \frac{\alpha_s}{2}\,\tnorm{x_s\!-\!x}^2$ and $\psi_0(x) = \tfrac{1/\eta_1}{2}\tnorm{x_1\!-\!x}^2$ so that $x_1 = \argmin_{x\in\cK} \psi_0(x)$. Note that $\psi_t$ are $(1/\eta_t)$-strongly convex while $\psi_0$ is $(1/\eta_1)$-strongly convex. Hence, it holds that
    \begin{align*}
        \tsum_{t=1}^T \pair{w_t \nabla f_t(x_t), \wtx_t - u}
        &\le \tsum_{t=1}^T \tfrac{\alpha_t}{2}\tnorm{x_t - u}^2 +  \tsum_{t=1}^T \tfrac{\eta_{t-1} w_t^2 \tnorm{\nabla f_t(x_t)}^2}{2}.
    \end{align*}
    To control $\tnorm{x_t - \wtx_t}$, apply Lemma~\ref{lem:stability-lemma} (with appropriate $\phi_0 = \phi_1 = \psi_{t-1}$),
    \begin{align*}
        \tnorm{x_t - \wtx_t} \le \eta_{t-1}\, \tnorm{\tsum_{s\in \cB_t} w_s \nabla f_s(x_s)}.
    \end{align*}
    Combining these results and Assumption~\ref{assump:gradient-norm-bound},  we can write
    \begin{align*}
         R^{w}_T(u) \le \tsum_{t=1}^T \tfrac{\alpha_t - w_t\lambda}{2}\,\tnorm{x_t-u}^2 + G^2\, \tsum_{t=1}^T \eta_{t-1} (w_t^2 + w_t \tsum_{s\in \cB_t} w_s).
    \end{align*}
    By substituting $\lambda = 0$ or $\lambda > 0$, we obtain the claimed results in both cases.
\end{proof}

\restatetheorem{thm:DW-BCO}{
    For all $u \in \cK$, predictions $(x_t)_{t=1}^T$ generated by DW-FTBL~\eqref{eq:DW-FTBL} satisfy
    \begin{align*}
        \widebar R^w_T(u) &\le \tfrac{D^2}{2\eta_T}
        + 2G^2\tsum_{t=1}^T\!\eta_{t-1}\!\rbr{\tfrac{k^2 M^2 w_t^2}{G^2 \delta_t^2 } + \tfrac{kM w_t}{G \delta_t} \sqrt{\tsum_{s\in \cB_t} w_s^2} + w_t \tsum_{s\in \cB_t} w_s} + \tfrac{3GD}{r}\tsum_{t=1}^T w_t\delta_t,
    \end{align*}
    where $\eta_0 = \eta_1$. Moreover, under $\lambda$-strong convexity (\ref{assump:strong-convexity}),
    \begin{align*}
        \widebar R^w_T(u) &\le \E\!\sbr{\tsum_{t=1}^T \tfrac{\alpha_t - w_t \lambda}{2}  \tnorm{y_t - u}^2} \\
        &+ 2G^2\tsum_{t=1}^T \eta_{t-1}\!\rbr{\tfrac{k^2 M^2 w_t^2}{G^2 \delta_t^2 } + \tfrac{kM w_t}{G \delta_t} \sqrt{\tsum_{s\in \cB_t} w_s^2} + w_t \tsum_{s\in \cB_t} w_s} + \tfrac{10 G^2}{\lambda r}\tsum_{t=1}^T w_t\delta_t.
    \end{align*}
}

\begin{proof}
    Consider a sequence $(\wty_t)_{t=1}^T$, defined by $\wty_1 = y_1$ and
    \begin{align*}
        \wty_{t+1} &= \argmin_{y\in \cK}\,\,  \pair{y,\, \tsum_{s \in [t]} w_s\, \htg_s} + \tsum_{s=1}^t \frac{\alpha_s}{2}\,\tnorm{y_s - y}^2.
    \end{align*}
    By induction, it follows simultaneously that, for every $t \in [T]$, (1) $y_t$ is $\cF_t$-measurable, (2) $\wty_t$ is $\cF_t$-measurable, and (3) $x_t$ and $\htg_t$ are $\cF_{t+1}$-measurable.

    Fix comparator $v \in \cK$. For $t\in [T]$, let $\vdelta_t = (1-\delta_t / r) v$ which satisfies $\tnorm{\vdelta_t - v} = \tfrac{\delta_t \norm{v}}{r}$. We treat the general convex case as 0-strongly convex. For $\lambda \ge 0$, the $\lambda$-strong convexity of $f_t$ gives
    \begin{align}
        \widebar{R}^w_T(v)
        &= \E\!\sbr{\tsum_{t=1}^{T} w_t(f_t(x_t) - f_t(v))}\notag\\
        &= \E\!\sbr{\tsum_{t=1}^{T} w_t(f^{\delta}_t(\ydelta_t) - f^{\delta}_t(\vdelta_t))} \notag\\
        &+ \E\!\sbr{\tsum_{t=1}^{T} w_t(f_t(x_t) - f_t(\ydelta_t))} + \E\!\sbr{\tsum_{t=1}^{T} w_t(f_t(\ydelta_t) - \fdelta_t(\ydelta_t))}\notag\\
        &+ \E\!\sbr{\tsum_{t=1}^{T} w_t(\fdelta_t(\vdelta_t) - f_t(\vdelta_t))} + \E\!\sbr{\tsum_{t=1}^{T} w_t(f_t(\vdelta_t) - f_t(v))}\notag\\
        &\myle{a} \E\!\sbr{\tsum_{t=1}^{T} w_t(\fdelta_t(\ydelta_t) - \fdelta_t(\vdelta_t))} + \tsum_{t=1}^T w_t\rbr{G\delta_t + G\delta_t  + G\delta_t  + \tfrac{G\norm{v}}{r}\delta_t}\notag\\
        &\myle{b} \E\!\sbr{\tsum_{t=1}^{T} w_t\rbr{\pair{\nabla \fdelta_t(\ydelta_t),\, \ydelta_t - \vdelta_t} - \tfrac{\lambda}{2} \tnorm{\ydelta_t - \vdelta_t}^2}} + \tfrac{2GD}{r} \tsum_{t=1}^T w_t \delta_t \notag\\
        &\myeq{c} \E\!\sbr{\tsum_{t=1}^{T} w_t\rbr{(1-\tfrac{\delta_t}{r})\,\pair{\nabla \fdelta_t(\ydelta_t),\, y_t - v} - \tfrac{\lambda}{2} (1-\tfrac{\delta_t}{r})^2\, \tnorm{y_t - v}^2}}  + \tfrac{2GD}{r} \tsum_{t=1}^T w_t \delta_t \notag\\
        &= \E\!\sbr{\tsum_{t=1}^{T} w_t \rbr{\pair{\nabla \fdelta_t(\ydelta_t),\, \wty_t - v} - \tfrac{\lambda}{2} \tnorm{y_t - v}^2}} \notag\\
        &+ \E\!\sbr{\tsum_{t=1}^{T} w_t \rbr{\pair{\nabla \fdelta_t(\ydelta_t),\, y_t - \wty_t}}} \notag\\
        &+ \E\!\sbr{\tsum_{t=1}^T w_t(\tfrac{\delta_t}{r} \pair{\nabla \fdelta_t(\ydelta_t),\, v - y_t} + (\tfrac{\lambda \delta_t}{r} - \tfrac{\lambda \delta_t^2}{2r^2})\,  \tnorm{y_t - v}^2 )} + \tfrac{2GD}{r} \tsum_{t=1}^T w_t \delta_t.\label{eq:BCO-proof-main}
    \end{align}
    where (a) uses $G$-Lipschitzness of $f_t$ and Theorem~\ref{thm:SPGE} to bound the last four terms in the previous line, (b) uses $\lambda$-strong convexity of $f^{\delta}_t$ and $3G\delta_t + \tfrac{G\norm{v}}{r}\delta_t \le \tfrac{2GD}{r}\delta_t$, since $r\Ball\subseteq\cK$ and the diameter bound give $\norm{v}+r\le D$ and $2r\le D$, and (c) substitutes $\ydelta_{t} = (1-\tfrac{\delta_t}{r}) y_{t}$ and $\vdelta_t = (1-\tfrac{\delta_t}{r}) v$.

    We analyze the terms in \eqref{eq:BCO-proof-main} separately. First, using the fact that $\E[\htg_t | \cF_t] = \nabla \fdelta_t(\ydelta_t)$ while $\wty_t$ is $\cF_t$-measurable, we write
    \begin{align*}
        \E\Big[\tsum_{t=1}^{T} w_t \Big(\pair{\nabla \fdelta_t(\ydelta_t),\, \wty_t - v} &- \tfrac{\lambda}{2} \tnorm{y_t - v}^2\Big)\Big]\\
        &= \E\!\sbr{\tsum_{t=1}^{T}  \rbr{\pair{w_t\htg_t,\, \wty_t - v} - \tfrac{w_t \lambda}{2} \tnorm{y_t - v}^2}}\\
        &\myle{d} \E\!\sbr{\tsum_{t=1}^T \tfrac{\alpha_t - w_t \lambda}{2}  \tnorm{y_t - v}^2 +  \tsum_{t=1}^T \tfrac{\eta_{t-1} w_t^2 \tnorm{\htg_t}^2}{2}}\\
        &\le \E\!\sbr{\tsum_{t=1}^T \tfrac{\alpha_t - w_t \lambda}{2}  \tnorm{y_t - v}^2} + G^2\tsum_{t=1}^T \eta_{t-1} \tfrac{k^2M^2w_t^2}{G^2\delta_t^2},
    \end{align*}
    where (d) applies Theorem~\ref{thm:FTRL-orabona} to $(\wty_t)_{t=1}^T$.

    Secondly, applying Assumption~\ref{assump:gradient-norm-bound} and Theorem~\ref{thm:SPGE}, we write
    \begin{align*}
        \E\Big[\tsum_{t=1}^{T} w_t (\pair{\nabla \fdelta_t(\ydelta_t),\, y_t - \wty_t})\Big]
        &\le G\,\, \E\!\sbr{\tsum_{t=1}^T w_t \tnorm{y_t - \wty_t}}\\
        &\myle{e} G\,\, \E\!\sbr{\tsum_{t=1}^T \eta_{t-1} w_t\, \tnorm{\tsum_{s\in \cB_t} w_s \htg_s}}\\
        &\myle{f} 2G\, \tsum_{t=1}^T \eta_{t-1} w_t \rbr{G\, \tsum_{s\in \cB_t} w_s + \tfrac{kM}{\delta_t} \sqrt{\tsum_{s\in \cB_t} w_s^2}}\\
        &= 2G^2\, \tsum_{t=1}^T \eta_{t-1} \rbr{\tfrac{kM w_t}{G\delta_t} \sqrt{\tsum_{s\in \cB_t} w_s^2} + w_t\tsum_{s\in \cB_t} w_s},
    \end{align*}
    where (e) applies Lemma~\ref{lem:stability-lemma} (with appropriate $\phi_0 = \phi_1$) and (f) uses Lemma~\ref{lem:BCO-meandering} on the set $\cB_t$.

    For the last term, it holds that
    \begin{align*}
        \E\!\sbr{\tsum_{t=1}^T w_t\rbr{\tfrac{\delta_t}{r} \pair{\nabla \fdelta_t(\ydelta_t),\, v - y_t} + (\tfrac{\lambda \delta_t}{r} - \tfrac{\lambda \delta_t^2}{2r^2})\,  \tnorm{y_t - v}^2 }} \le \tfrac{(G+\lambda D)D}{r} \tsum_{t=1}^T w_t \delta_t.
    \end{align*}

    Combining these bounds, we have
    \begin{align*}
        \widebar R_T(v) &\le \E\!\sbr{\tsum_{t=1}^T \tfrac{\alpha_t - w_t \lambda}{2}  \tnorm{y_t - v}^2} + G^2\tsum_{t=1}^T \eta_{t-1} \tfrac{k^2M^2w_t^2}{G^2\delta_t^2}\\
        &+ 2G^2\, \tsum_{t=1}^T \eta_{t-1} \rbr{\tfrac{kM w_t}{G\delta_t} \sqrt{\tsum_{s\in \cB_t} w_s^2} + w_t\tsum_{s\in \cB_t} w_s}\\
        &+ \tfrac{(G+\lambda D)D}{r} \tsum_{t=1}^T w_t \delta_t + \tfrac{2GD}{r} \tsum_{t=1}^T w_t \delta_t\\
        &\le \E\!\sbr{\tsum_{t=1}^T \tfrac{\alpha_t - w_t \lambda}{2}  \tnorm{y_t - v}^2}\\
        &+ 2G^2\, \tsum_{t=1}^T \eta_{t-1} \rbr{\tfrac{k^2M^2w_t^2}{G^2\delta_t^2}+ \tfrac{kM w_t}{G\delta_t} \sqrt{\tsum_{s\in \cB_t} w_s^2} + w_t\tsum_{s\in \cB_t} w_s}\\
        &+ \tfrac{3GD + \lambda D^2}{r} \tsum_{t=1}^T w_t \delta_t.
    \end{align*}
    By substituting $\lambda = 0$ or $\lambda > 0$, while letting $D \le \frac{2G}{\lambda}$ in the strongly convex case (see Fact~\ref{fact:sc-diameter}), we obtain the claimed results in both cases.
\end{proof}

\begin{lemma}\label{lem:BCO-meandering}
    Under the conditions of Theorem~\ref{thm:DW-BCO}, for every set $\cB \subseteq [T]$ with $\delta = \min_{t\in \cB}\delta_t$,
    \begin{align*}
        \E\!\sbr{\tnorm{\tsum_{t\in \cB} w_t \htg_t}} \le 2\rbr{G\, \tsum_{t\in \cB} w_t + \tfrac{kM}{\delta}\sqrt{\tsum_{t\in \cB} w_t^2}}.
    \end{align*}
\end{lemma}
\begin{proof}
    For every $t \in \cB$, let $g_t = \nabla \fdelta_t(\ydelta_t) = \E[\htg_t|\cF_t]$ and note
    \begin{align*}
        \tnorm{g_t} \le G \qquad \text{and} \qquad \tnorm{\htg_t} \le \tfrac{kM}{\delta}.
    \end{align*}
    Consequently, from the triangle inequality and Jensen's inequality, it follows that
    \begin{align*}
        \E\!\sbr{\tnorm{\tsum_{t\in \cB} w_t \htg_t}}
        &\le \E\!\sbr{\tnorm{\tsum_{t\in \cB} w_t g_t}} +  \E\!\sbr{\tnorm{\tsum_{t\in \cB} w_t (\htg_t - g_t)}}\\
        &\le G\tsum_{t\in \cB} w_t + \sqrt{\E\!\sbr{\tnorm{\tsum_{t\in \cB} w_t (\htg_t - g_t)}^2}}.
    \end{align*}
    The second term under the square root can be bounded as
    \begin{align*}
        \E\!\sbr{\norm{\tsum_{t\in \cB} w_t (\htg_t - g_t)}^2}
        &= \E\!\sbr{\tsum_{t\in\cB} w_t^2 \tnorm{\htg_t - g_t}^2}\\
        &+ 2\,\E\!\sbr{\tsum_{t=1}^T \tsum_{s=t+1}^T \ind(t,s\in \cB)\, w_t w_s\, \pair{\htg_t - g_t, \htg_s - g_s} }\\
        &\le \tsum_{t\in \cB} w_t^2 (G + \tfrac{kM}{\delta})^2 + 0.
    \end{align*}
    Combining these results, we write
    \begin{align*}
        \E\!\sbr{\tnorm{\tsum_{t\in \cB} w_t \htg_t}} &\le G\, \tsum_{t\in \cB} w_t + (G + \tfrac{kM}{\delta}) \sqrt{\tsum_{t\in \cB} w_t^2}\\
        &\le 2\rbr{G\tsum_{t\in \cB} w_t + \tfrac{kM}{\delta} \sqrt{\tsum_{t\in \cB} w_t^2}}.
    \end{align*}
    This concludes the proof of the lemma.
\end{proof}

\section{Reduction via Proxy Delay Scheduling: Proofs}\label{app:proofs-reduction}

We prove the results of \Cref{sec:reduction}. The proof of \Cref{thm:wrapper-regret} extends the importance-weighting argument of \citet{capacity_constraint2025} from multi-armed bandits to the convex setting. The saturation bound (Lemma~\ref{lem:proxy-condition}) and the Pareto scheduler (Corollary~\ref{cor:pareto-scheduler}) simplify the corresponding analysis by reducing it to a single Chernoff condition on the tracking set size. The Bernoulli scheduler (Corollary~\ref{cor:bernoulli-sigmax}), which calibrates to the maximum backlog~$\sigmax$, is new to this work.

\restatetheorem{thm:wrapper-regret}{
    Suppose the proxy delays are sampled independently and
    $\Pr(\wtd_t \ge d_t)>0$ for every $t\in[T]$.
    The weights forwarded to $\gA$ satisfy
\[
    w_t = \tfrac{\ind(|\cS_t| < C)\,\ind(\wtd_t \ge d_t)}{\Pr(\wtd_t \ge d_t)}, \quad \E[w_t] \le 1,\quad \E[w_t^2] \le \tfrac{1}{\Pr(\wtd_t \ge d_t)}, \quad \E[w_t w_s] \le 1,
\]
for all distinct $t,s \in [T]$. Moreover, the expected regret satisfies
\[
    \widebar R_T(u) \;\le\; \E\!\sbr{\tsum_{t=1}^T w_t\bigl(f_t(x_t) - f_t(u)\bigr)} + GD\, \tsum_{t=1}^T \Pr\bigl(|\cS_t| = C\bigr).
\]
}
\begin{proof}
    From the structure of Algorithm~\ref{alg:scheduler-proxy-delays}, feedback from round $t$ is observed in round $t+d_t$ if and only if $|S_t| < C$ and $\wtd_t \ge d_t$. Consequently, we have the identity $w_t = \tfrac{\ind(|\cS_t| < C)\, \ind(\wtd_t \ge d_t)}{\Pr(\wtd_t \ge d_t)}$.

    Therefore, for each $t\in [T]$ and $s \in [T]\setminus \{t\}$, it holds that
    \begin{align*}
        \E[w_t] \le \E\sbr{\tfrac{\ind(\wtd_t \ge d_t)}{\Pr(\wtd_t \ge d_t)}} = 1,\quad \E[w_t^2] \le \E\sbr{\tfrac{\ind(\wtd_t \ge d_t)}{(\Pr(\wtd_t \ge d_t))^2}} = \tfrac{1}{\Pr(\wtd_t \ge d_t)},\\
        \E[w_t w_s] \le \E\sbr{\tfrac{\ind(\wtd_t \ge d_t)\,\ind(\wtd_s \ge d_s)}{\Pr(\wtd_t \ge d_t)\, \Pr(\wtd_s \ge d_s)}} = \E\sbr{\tfrac{\ind(\wtd_t \ge d_t)}{\Pr(\wtd_t \ge d_t)}}\, \E\sbr{\tfrac{\ind(\wtd_s \ge d_s)}{\Pr(\wtd_s \ge d_s)}} = 1,
    \end{align*}
    where the last inequality applies the fact that proxy delays $\wtd_t$ and $\wtd_s$ are sampled independently.

    For the regret bound, we write
    \begin{align*}
        \widebar R_T(u)
        &= \E\!\sbr{\tsum_{t=1}^T (f_t(x_t) - f_t(u))}\\
        &\myle{a} \E\!\sbr{\tsum_{t=1}^T \ind(|\cS_t| < C)\,(f_t(x_t) - f_t(u))} + GD\,\E\!\sbr{\tsum_{t=1}^T \ind(|\cS_t| = C)}\\
        &\myeq{b} \E\!\sbr{\tsum_{t=1}^T \tfrac{\ind(|\cS_t| < C)\, \ind(\wtd_t \ge d_t)}{\Pr(\wtd_t \ge d_t)}\,(f_t(x_t) - f_t(u))} + GD\, \tsum_{t=1}^T \Pr(|\cS_t| = C)\\
        &\myeq{c} \E\!\sbr{\tsum_{t=1}^T w_t\,(f_t(x_t) - f_t(u))} + GD\, \tsum_{t=1}^T \Pr(|\cS_t| = C),
    \end{align*}
    where (a) combines Assumptions~\ref{assump:finite-diameter} and \ref{assump:gradient-norm-bound}, (b) uses the fact that $\wtd_t$ is sampled from $\gD_t$ at the start of round $t$ independently of $S_t$ and $x_t$, and (c) applies the identity $w_t = \tfrac{\ind(|\cS_t| < C)\, \ind(\wtd_t \ge d_t)}{\Pr(\wtd_t \ge d_t)}$.
\end{proof}

\restatelemma{lem:proxy-condition}{
    Consider proxy delay distributions $\gD = (\gD_t)_{t=1}^T$ such that for every $t \in [T]$
    \begin{align*}
        \tsum_{s \in \cB_t \cup \{t\}} \Pr(\wtd_s \ge t-s) \le \frac{C}{8}.
    \end{align*}
    Then, $\Pr\bigl(|\cS_t| = C\bigr) \le \exp(-C)$ for every $t \in [T]$.
}

\begin{proof}
Fix $t \in [T]$. For each $s \in \cB_t \cup \{t\}$, define $X_s = \ind(\wtd_s \ge t-s)$.

By independence of the proxy-delay draws across rounds, the random variables $(X_s)_{s \in \cB_t \cup \{t\}}$ are independent Bernoulli random variables. Moreover, by construction of the proxy scheduler, $|\cS_t| \le \sum_{s \in \cB_t \cup \{t\}} X_s$. Therefore, writing $X = \sum_{s \in \cB_t \cup \{t\}} X_s$,
we have
\begin{align*}
    \Pr(|\cS_t| = C) \le \Pr(|\cS_t| \ge C) \le \Pr(X \ge C).
\end{align*}

Let $\mu = \E[X] = \sum_{s \in \cB_t \cup \{t\}} \Pr(\wtd_s \ge t-s)$. By assumption, $\mu \le C/8$.

If $\mu = 0$, then $\Pr(X \ge C) = 0$. Assume henceforth that $\mu > 0$, and define
\begin{align*}
    \alpha = C/\mu - 1 \in [7, \infty).
\end{align*}
By the multiplicative Chernoff bound (e.g., see Theorem 2.3.1 from \citet{Vershynin_2026}),
\begin{align*}
    \Pr(X \ge C)
    &= \Pr\bigl(X \ge (1+\alpha)\mu\bigr) \\
    &\le \left( \frac{e^\alpha}{(1+\alpha)^{1+\alpha}} \right)^\mu \\
    &= \exp\left( C \left( \tfrac{\alpha}{1+\alpha} - \ln(1+\alpha) \right) \right).
\end{align*}
As the function $ f(x) = \ln(1+x) - \frac{x}{1+x}$ is increasing on $[0,\infty)$, $\alpha$ satisfies
\begin{align*}
    \tfrac{\alpha}{1+\alpha} - \ln(1+\alpha)
    \le 7/8 - \ln 8
    < -1.
\end{align*}
Therefore,
\begin{align*}
    \Pr\bigl(|\cS_t| = C\bigr) \le \Pr(X \ge C) \le e^{-C}.
\end{align*}
\end{proof}

\restatecorollary{cor:pareto-scheduler}{
    For each $t \in [T]$, let $H_t = \tsum_{j=1}^t \tfrac{1}{j}$. Take $\gD_t$ to be the shifted floored Pareto distribution
    \[
        \gD_t \;=\; \floor{\Pareto\bigl(\tfrac{C}{16H_t},\,1\bigr)} - 1.
    \]
    Then $\Pr(\wtd_t \ge d_t) = \min\bigl\{1,\, \tfrac{C}{16 H_t (d_t+1)}\bigr\}$ and $\Pr(|\cS_t| = C) \le \exp(-C)$ for every $t \in [T]$.
}
\begin{proof}
Fix $t \in [T]$. For each $s\le t$, by the Pareto tail formula,
\begin{align*}
    \Pr(\wtd_s \ge t-s)
    &= \Pr\!\left(\Pareto\!\left(\tfrac{C}{16H_s},1\right) \ge t-s + 1\right) = \min\cbr{1, \tfrac{C}{16H_s(t-s+1)}}.
\end{align*}
It remains to verify the condition of Lemma~\ref{lem:proxy-condition}. It follows from Fact~\ref{fact:harmonic} for harmonic numbers.

Thus, Lemma~\ref{lem:proxy-condition} now yields $\Pr(|\cS_t| = C) \le e^{-C}$.
\end{proof}

\restatecorollary{cor:bernoulli-sigmax}{
    Let $p = \min\{1, \tfrac{C}{8(\sigmax+1)}\}$. For every $t \in [T]$, let the proxy distribution $\gD_t \in \Delta(\{-1, \infty\})$ be given by
    \begin{align*}
        \textstyle \Pr_{\wtd \sim \gD_t}(\wtd = \infty) = p,
        \qquad
        \Pr_{\wtd \sim \gD_t}(\wtd = -1) = 1-p.
    \end{align*}
    Then, $\Pr(\wtd_t \ge d_t) = p$ and $\Pr(|\cS_t| = C) \le \exp(-C)$ for every $t \in [T]$.
}
\begin{proof}
    Fix $t \in [T]$. Since $d_t \ge 0$ and $\wtd_t \in \{-1,\infty\}$, $\Pr(\wtd_t \ge d_t) = \Pr(\wtd_t = \infty) = p$.

    Thus, the condition of Lemma~\ref{lem:proxy-condition} is satisfied
    \begin{align*}
        \sum_{s \in \cB_t \cup \{t\}} \Pr(\wtd_s \ge t-s)
        = p \, |\cB_t \cup \{t\}| = p (\sigma_t + 1) \le \tfrac{C (\sigma_t + 1)}{8 (\sigmax + 1)} \le C/8.
    \end{align*}
    Applying Lemma~\ref{lem:proxy-condition}, we obtain
    \begin{align*}
        \Pr(|\cS_t| = C) \le \exp(-C).
    \end{align*}
    This completes the proof.
\end{proof}

\section{Guarantees for Capacity-Constrained OCO and BCO: Proofs}\label{app:proofs-oco-bco}

In this section, we prove the main regret bounds for capacity-constrained OCO and BCO stated in \Cref{sec:results}, covering both convex and strongly convex losses.

\restatetheorem{thm:OCO-main-result}{
    Suppose $C \ge \ln T + 1$, and consider $\cW(\gA, \gD)$, where $\gA$ is DW-FTRL~\eqref{eq:DW-FTRL}. With learning rates $\eta_t = \tfrac{D/G}{\sqrt{\dtot + T}}$ and $\gD$ from Corollary~\ref{cor:pareto-scheduler}, $\cW(\gA, \gD)$ guarantees
    \begin{align*}
        \widebar R_T(x^*) = O\rbr{GD\,{\sqrt{\dtot + T}}}.
    \end{align*}
    Under $\lambda$-strong convexity (\ref{assump:strong-convexity}), for rates $\eta_t = \frac{1}{t\lambda}$ and $\gD$ from Corollary~\ref{cor:bernoulli-sigmax}, $\cW(\gA, \gD)$ guarantees
    \begin{align*}
        \widebar R_T(x^*) = O\rbr{\tfrac{G^2}{\lambda} (\sigmax + 1) \ln T}.
    \end{align*}
}

\begin{proof}
    Note that $C \ge \ln(eT) \ge \sum_{t=1}^T \frac{1}{t} = H_T$ and $T e^{-C} \le 1$ throughout the proof. Fix comparator $u \in \cK$. We prove the bounds on $\widebar R_T(u)$ for general convex and strongly convex losses separately.

    \textbf{General convex losses:} We use fixed learning rate $\eta_t = \eta$ for the value $\eta = \tfrac{D/G}{\sqrt{\dtot + T}}$. Combining Theorems~\ref{thm:wrapper-regret}, \ref{thm:DW-OCO} and Corollary~\ref{cor:pareto-scheduler}, we write
    \begin{align*}
        \widebar R_T(u)
        &\le \E\!\sbr{\tsum_{t=1}^T w_t (f_t(x_t) - f_t(u))} + GD\, \tsum_{t=1}^T \Pr(|\cS_t| = C)\\
        &\le \tfrac{D^2}{\eta} + G^2\eta\,\E\!\sbr{\tsum_{t=1}^T (w_t^2 + w_t\tsum_{s\in \cB_t} w_s)} + GD\,T\exp(-C)\\
        &\le \tfrac{D^2}{\eta} + G^2\eta\,\tsum_{t=1}^T \rbr{\max\cbr{1, \tfrac{16 H_t}{C}\,(d_t + 1)}+ \sigma_t} + GD\\
        &\myle{a} \tfrac{D^2}{\eta} + G^2\eta\, \rbr{1 + \tfrac{16 H_T}{C}}\, \rbr{\dtot + T} + GD\\
        &\myle{b} \tfrac{D^2}{\eta} + 17 G^2\eta\,\rbr{\dtot + T} + GD\\
        &= 18GD\,\sqrt{\dtot + T} + GD,
    \end{align*}
    where (a) applies $\sum_{t=1}^T \sigma_t = \dtot$ from Lemma~\ref{lem:backlog-properties} and (b) follows for $C \ge H_T$.

    \textbf{Strongly convex losses:} Combining Theorems~\ref{thm:wrapper-regret}, \ref{thm:DW-OCO} and Corollary~\ref{cor:bernoulli-sigmax}, while letting $D \le \frac{2G}{\lambda}$ in the strongly convex case (see Fact~\ref{fact:sc-diameter}), we write
    \begin{align}
        \widebar R_T(u)
        &\le \E\!\sbr{\tsum_{t=1}^T w_t (f_t(x_t) - f_t(u))} + GD\, \tsum_{t=1}^T \Pr(|\cS_t| = C)\notag\\
        &\le \E\!\sbr{\tsum_{t=1}^T \tfrac{\alpha_t - w_t\lambda}{2}\,\tnorm{x_t\!-\!u}^2} + G^2\, \E\!\sbr{\tsum_{t=1}^T \eta_{t-1} (w_t^2\!+\!w_t \tsum_{s\in \cB_t}\!w_s)} + \tfrac{2G^2}{\lambda}\,Te^{-C}\notag\\
        &\le \E\!\sbr{\tsum_{t=1}^T \tfrac{\lambda - w_t\lambda}{2}\,\tnorm{x_t-u}^2} + 2G^2\, \E\!\sbr{\tsum_{t=1}^T \eta_t (w_t^2 + w_t \tsum_{s\in \cB_t} w_s)} + \tfrac{2G^2}{\lambda},\label{eq:cor-OCO-sc-main}
    \end{align}
    where the last inequality substitutes $\alpha_t = \lambda$, $\eta_{t-1} \le 2 \eta_t$, and $T e^{-C} \le 1$.

    To control the first term in \eqref{eq:cor-OCO-sc-main}, we use the identity for $w_t$ in Theorem~\ref{thm:wrapper-regret},
    \begin{align*}
        \E\!\sbr{\tsum_{t=1}^T \tfrac{\lambda - w_t\lambda}{2}\,\tnorm{x_t-u}^2}
        &= \tfrac{\lambda}{2}\tsum_{t=1}^T \E\!\sbr{\Big(1 - \tfrac{\ind(|\cS_t|<C)\, \ind(\wtd_t \ge d_t)}{\Pr(\wtd_t \ge d_t)}\Big) \tnorm{x_t-u}^2}\\
        &\myeq{a} \tfrac{\lambda}{2}\tsum_{t=1}^T \E\!\sbr{(1-\ind(|\cS_t| < C))\, \tnorm{x_t-u}^2}\\
        &\myle{b} \tfrac{2G^2}{\lambda} \tsum_{t=1}^T \Pr(|\cS_t| = C)\\
        &\le \tfrac{2G^2}{\lambda}\, T e^{-C}\\
        &\le \tfrac{2G^2}{\lambda},
    \end{align*}
    where (a) uses the fact that $d'_t$ is sampled independently of $x_t$, and (b) lets $D \le \frac{2G}{\lambda}$ (see Fact~\ref{fact:sc-diameter}).

    For the second term in \eqref{eq:cor-OCO-sc-main}, we apply $\E[w_t^2] \le \frac{1}{p}$ and $\E[w_s w_t] \le 1$ from Theorem~\ref{thm:wrapper-regret} for $p = \min\{1, \frac{C}{8(\sigmax + 1)}\}$ in Corollary~\ref{cor:bernoulli-sigmax} and substitute learning rate values $\eta_t = \frac{1}{t\lambda}$,
    \begin{align*}
        \E\!\sbr{\tsum_{t=1}^T \eta_t (w_t^2 + w_t \tsum_{s\in \cB_t} w_s)}
        &\le \tsum_{t=1}^T \tfrac{1}{t\lambda} \rbr{\max\cbr{1, \frac{8(\sigmax+1)}{C}}+\sigmax}\\
        &\le \tfrac{9 (\sigmax+1) H_T}{\lambda}.
    \end{align*}
    Combining these bounds, we finally have
    \begin{align*}
        \widebar R_T(u) \le \tfrac{18 G^2}{\lambda} (\sigmax+1) H_T + \tfrac{4G^2}{\lambda} = O\rbr{\tfrac{G^2}{\lambda}(\sigmax + 1)\ln T}.
    \end{align*}
    This completes the proof.
\end{proof}

We now proceed to the BCO case.
\restatetheorem{thm:BCO-main-result}{
    Suppose $C \ge \ln T + 1$, and consider $\cW(\gA, \gD)$, where $\gA$ is DW-FTBL~\eqref{eq:DW-FTBL} and $\gD$ is from Corollary~\ref{cor:bernoulli-sigmax}. With smoothing parameters $\delta_t = r\,\min\cbr{1,\, \tfrac{\sqrt{\nu k}}{t^{1/4}}\, (1 + \tfrac{\sigmax}{C})^{1/4}}$ and learning rates $\eta_t = \tfrac{D/G}{ t^{1/2}\sqrt{\sigmax + (\frac{\nu k}{\delta_t/r})^2 (1+\frac{\sigmax}{C})}}$, $\cW(\gA, \gD)$ guarantees
    \begin{align*}
        \widebar R_T(x^*) = O\!\rbr{GD\rbr{\sqrt{T\sigmax} + T^{3/4}\rbr{1 + \tfrac{\sigmax}{C}}^{1/4}\!\sqrt{\nu k}}}.
    \end{align*}
    Under $\lambda$-strong convexity (\ref{assump:strong-convexity}), with $\eta_t = \tfrac{1}{t\lambda}$ and $\delta_t = r\,\min\cbr{1,\, (\tfrac{\nu^2 k^2 \ln(t+1)}{t})^{1/3} ({1+ \tfrac{\sigmax}{C}})^{1/3}}$,
    \begin{align*}
        \widebar R_T(x^*) = O\!\Big(\tfrac{G^2}{\lambda} \Big(\sigmax\ln T + (T^2 \ln T)^{1/3} (1 + \tfrac{\sigmax}{C})^{1/3} (\nu k)^{2/3}\Big)\Big).
    \end{align*}
}
\begin{proof}
    Recall $\nu = \frac{M}{Gr}$. Note that $C \ge \ln(eT) \ge \sum_{t=1}^T \frac{1}{t}$ and $T e^{-C} \le 1$ throughout the proof. Fix comparator $u \in \cK$. We prove the bounds on $\widebar R_T(u)$ for convex and strongly convex losses separately.

    \textbf{General convex losses:} We use non-increasing smoothing parameters and learning rates
    \begin{align*}
        \delta_t = r\,\min\cbr{1, \tfrac{\sqrt{\nu k}}{t^{1/4}}\, (1\!+\!\tfrac{\sigmax}{C})^{1/4}}, \qquad \eta_t = \tfrac{D/G}{ t^{1/2}\sqrt{\sigmax + (\frac{\nu k}{\delta_t/r})^2 (1+\frac{\sigmax}{C})}}.
    \end{align*}
    Observe that $\delta_{t-1} \le 2\delta_t$ for all $t \ge 2$ and consequently $\eta_{t-1} \le 2\eta_t$ for all $t \ge 1$.

    Also, consider $\deltatot = \tsum_{t=1}^T \delta_t$, for which it holds that
    \begin{align*}
        \deltatot/r \le \rbr{\tsum_{t=1}^T t^{-1/4}} (1\!+\!\tfrac{\sigmax}{C})^{1/4} \sqrt{\nu k} \le \tfrac{4}{3} T^{3/4} (1\!+\!\tfrac{\sigmax}{C})^{1/4} \sqrt{\nu k}.
    \end{align*}

    From Corollary~\ref{cor:bernoulli-sigmax}, we additionally get
    \begin{align*}
        1/p = \max\bigl\{1,\, \tfrac{8(\sigmax+1)}{C}\bigr\} \le 9(1+\tfrac{\sigmax}{C}).
    \end{align*}

    Combining Theorems~\ref{thm:wrapper-regret}, \ref{thm:DW-BCO} and Corollary~\ref{cor:bernoulli-sigmax}, we write
    \begin{align}
         \widebar R_T(u)
         &\le \E\!\sbr{\tsum_{t=1}^T w_t (f_t(x_t) - f_t(u))} + GD\, \tsum_{t=1}^T \Pr(|\cS_t| = C)\notag\\
         &\le \tfrac{D^2}{\eta_T}
        + 2G^2\, \tsum_{t=1}^T \eta_{t-1}\E\!\sbr{(\tfrac{kM w_t}{G \delta_t})^2 + \tfrac{kM w_t}{G \delta_t} \sqrt{\tsum_{s\in \cB_t} w_s^2} + w_t \tsum_{s\in \cB_t} w_s}\notag\\
        &+ \tfrac{3GD}{r}\tsum_{t=1}^T \E[w_t\delta_t] + GD\,Te^{-C}\notag\\
        &\myle{a} \tfrac{D^2}{\eta_T} + 4G^2\, \tsum_{t=1}^T \eta_t \rbr{(\tfrac{k \nu}{\delta_t/r})^2/p + \tfrac{k\nu}{\delta_t/r} \sqrt{\sigmax/p} + \sigmax} + 3GD\,(\tfrac{\deltatot}{r} + 1)\notag\\
        &\myle{b} \tfrac{D^2}{\eta_T} + 72 G^2 \tsum_{t=1}^T \eta_t \Big(\sigmax + (\tfrac{k \nu}{\delta_t/r})^2 (1+\tfrac{\sigmax}{C}) \Big) + 3GD\,(\tfrac{\deltatot}{r} + 1)\notag\\
        &\myle{c} 145GD\,\rbr{\sqrt{T\sigmax} +  T^{1/2} (\tfrac{k \nu}{\delta_T/r}) (1+\tfrac{\sigmax}{C})^{1/2} + \deltatot/r + 1} ,\label{eq:bco-thm-step}
    \end{align}
    where (a) uses Jensen's inequality for $\E[({\sum_{s\in\cB_t} w_s^2})^{1/2}]^2 \le \E[\sum_{s\in\cB_t} w_s^2] \le \sigmax/p$, (b) uses AM-GM and $1/p \le 9(1+\frac{\sigmax}{C})$, and (c) substitutes $\eta_t$ and uses subadditivity of the square root.

    Let $\delta_T < r$, since otherwise $T^{1/4} \le (1+\tfrac{\sigmax}{C})^{1/4}\sqrt{\nu k}$, and the bound is vacuous $O(T)$. Then,
    \begin{align*}
        T^{1/2} (\tfrac{k \nu}{\delta/r}) (1+\tfrac{\sigmax}{C})^{1/2}
        &= T^{3/4} (1+\tfrac{\sigmax}{C})^{1/4}\sqrt{\nu k}, \qquad \deltatot/r \le \tfrac{4}{3} T^{3/4} (1\!+\!\tfrac{\sigmax}{C})^{1/4} \sqrt{\nu k}.
    \end{align*}

    Combining the bounds above for $u = x^*$ yields
    \begin{align*}
        \widebar R_T(x^*) = O\rbr{GD\rbr{\sqrt{T\sigmax} + T^{3/4} (1+\tfrac{\sigmax}{C})^{1/4}\sqrt{\nu k}}}.
    \end{align*}

    \textbf{Strongly convex losses:} We use non-increasing smoothing parameters and learning rates
    \begin{align*}
        \delta_t = r\,\min\cbr{1,\, (\tfrac{\nu^2 k^2 \ln(t+1)}{t})^{1/3} ({1+ \tfrac{\sigmax}{C}})^{1/3}}, \qquad \eta_t = \tfrac{1}{\lambda t}.
    \end{align*}

    Consider $\deltatot = \tsum_{t=1}^T \delta_t$, for which it holds that
    \begin{align*}
        \deltatot/r &\le \rbr{\tsum_{t=1}^T t^{-1/3}} (\ln (T+1))^{1/3} (1\!+\!\tfrac{\sigmax}{C})^{1/3} (\nu k)^{2/3}\\
        &\le \tfrac{3}{2} (T^{2} \ln(T+1))^{1/3} (1\!+\!\tfrac{\sigmax}{C})^{1/3} (\nu k)^{2/3}.
    \end{align*}

    Combining Theorems~\ref{thm:wrapper-regret}, \ref{thm:DW-BCO} and Corollary~\ref{cor:bernoulli-sigmax}, we write
    \begin{align}
         \widebar R_T(u)
         &\le \E\!\sbr{\tsum_{t=1}^T w_t (f_t(x_t) - f_t(u))} + GD\, \tsum_{t=1}^T \Pr(|\cS_t| = C)\notag\\
         &\le \E\!\sbr{\tsum_{t=1}^T \tfrac{\alpha_t - w_t \lambda}{2}  \tnorm{y_t - u}^2} + \tfrac{10 G^2}{\lambda r} \E\sbr{\tsum_{t=1}^T w_t\delta_t}\notag\\
        &+ 4G^2\tsum_{t=1}^T \eta_{t} \E\sbr{\tfrac{k^2 M^2 w_t^2}{G^2 \delta_t^2 } + \tfrac{kM w_t}{G \delta_t} \sqrt{\tsum_{s\in \cB_t} w_s^2} + w_t \tsum_{s\in \cB_t} w_s} + GD,\label{eq:bco-thm-step-sc}
    \end{align}
    where the last inequality additionally applies $\eta_{t-1} \le 2\eta_t$.

    To control the first term in \eqref{eq:bco-thm-step-sc}, we use the identity for $w_t$ in Theorem~\ref{thm:wrapper-regret},
    \begin{align*}
        \E\!\sbr{\tsum_{t=1}^T \tfrac{\lambda - w_t\lambda}{2}\,\tnorm{y_t-u}^2}
        &= \tfrac{\lambda}{2}\tsum_{t=1}^T \E\!\sbr{\Big(1 - \tfrac{\ind(|\cS_t|<C)\, \ind(\wtd_t \ge d_t)}{\Pr(\wtd_t \ge d_t)}\Big) \tnorm{y_t-u}^2}\\
        &\myeq{d} \tfrac{\lambda}{2}\tsum_{t=1}^T \E\!\sbr{(1-\ind(|\cS_t| < C))\, \tnorm{y_t-u}^2}\\
        &\myle{e} \tfrac{2G^2}{\lambda} \tsum_{t=1}^T \Pr(|\cS_t| = C)\\
        &\le \tfrac{2G^2}{\lambda},
    \end{align*}
    where (d) uses the fact that $d'_t$ is sampled independently of $y_t$, and (e) lets $D \le \frac{2G}{\lambda}$ (see Fact~\ref{fact:sc-diameter}).

    For the second term in \eqref{eq:bco-thm-step-sc}, we use Theorem~\ref{thm:wrapper-regret} to write
    \begin{align*}
        \tfrac{G^2}{\lambda r}\E\sbr{\tsum_{t=1}^T w_t\delta_t} &\le  \tfrac{G^2}{\lambda}\, (\deltatot/r)\\
        &\le \tfrac{3G^2}{2\lambda}\,  T^{2/3} (\ln (T+1))^{1/3} ({1+ \tfrac{\sigmax}{C}})^{1/3} (\nu k)^{2/3}.
    \end{align*}
    For the third term in \eqref{eq:bco-thm-step-sc}, we follow the analysis of the general convex case in \eqref{eq:bco-thm-step} to write
    \begin{align*}
        G^2\tsum_{t=1}^T \eta_t &\E\sbr{\tfrac{k^2 M^2 w_t^2}{G^2 \delta_t^2 } + \tfrac{kM w_t}{G \delta_t} \sqrt{\tsum_{s\in \cB_t} w_s^2} + w_t \tsum_{s\in \cB_t} w_s}\\
        &\le \tfrac{18 G^2}{\lambda}\,\rbr{\tsum_{t=1}^T \tfrac{1}{t}} \Big(\sigmax + (\tfrac{k \nu}{\delta_T/r})^2 (1+\tfrac{\sigmax}{C}) \Big) \\
        &= O\rbr{\tfrac{G^2}{\lambda}\rbr{\sigmax \ln T + T^{2/3} (\ln T)^{1/3} ({1+ \tfrac{\sigmax}{C}})^{1/3} (\nu k)^{2/3}}}.
    \end{align*}
    Combining the bounds above for $u = x^*$ yields
    \begin{align*}
        \widebar R_T(x^*) = O\rbr{\tfrac{G^2}{\lambda}\rbr{\sigmax \ln T + T^{2/3} (\ln T)^{1/3} ({1+ \tfrac{\sigmax}{C}})^{1/3} (\nu k)^{2/3}}}.
    \end{align*}
    This completes the proof.
\end{proof}

\section{Doubling Trick over Maximum Backlog}\label{app:doubling-tricks}

The strongly convex and bandit results of \Cref{thm:OCO-main-result,thm:BCO-main-result} require tuning with knowledge of~$\sigmax$, which may not be available in advance. Since the learning rates and smoothing parameters in all our results depend on the round index~$t$ but not on the horizon~$T$, our base algorithms are anytime in~$T$, and we need only double over~$\sigmax$. This avoids the two-dimensional doubling over $(T, \dtot)$ used by \citet{bistritz2022}.

\paragraph{Epoch structure.}
Define backlog guesses $S_j = 2^j - 1$ for $j \in \Z_{\ge 0}$. The wrapper initializes at epoch $j = 0$ and restarts the base algorithm with parameters tuned to maximum backlog~$S_j$. Within each epoch, the algorithm runs until the observed backlog exceeds~$S_j$, at which point the epoch ends, $j$ is incremented, and the algorithm restarts with the updated guess. Upon starting a new epoch, the wrapper clears the tracking set, discarding any pending observations from the previous epoch. The final epoch index is $J = \tceil{\log_2(\sigmax + 1)}$.

\begin{theorem}[Doubling over $\sigmax$]\label{thm:doubling-sigmax}
    Suppose the base algorithm is anytime in~$T$: for any backlog guess~$\widetilde{S}$, on any interval of length~$T'$ with maximum backlog at most~$\widetilde{S}$, it guarantees regret at most
    \begin{align*}
        R(T', \widetilde{S})
        \;\le\;
        \varphi_0(T') + \tsum_{m=1}^M \varphi_m(T')\, (\widetilde{S}+1)^{b_m},
    \end{align*}
    where $b_m > 0$ and each $\varphi_m : \N \to [0, \infty)$ for $m \in \{0\} \cup [M]$ is non-decreasing. Then the doubling wrapper guarantees
    \begin{align*}
        R(T, \sigmax)
        \;=\;
        O\!\left(\varphi_0(T)\bigl(1+\log(\sigmax\!+\!1)\bigr) + \tsum_{m=1}^M \varphi_m(T)\, (\sigmax+1)^{b_m}\right).
    \end{align*}
\end{theorem}

\begin{proof}
The wrapper produces at most $J + 1$ epochs, indexed $j = 0, 1, \ldots, J$ with $J = \tceil{\log_2(\sigmax+1)}$. Let $L_j$ denote the number of rounds in epoch~$j$. Since $L_j \le T$ and each~$\varphi_m$ is non-decreasing, the regret decomposes as
\begin{align*}
    R(T, \sigmax)
    \;\le\; \tsum_{j=0}^J R(L_j, S_j)
    \;\le\; (J\!+\!1)\,\varphi_0(T) \;+\; \tsum_{m=1}^M \varphi_m(T) \tsum_{j=0}^J 2^{b_m j}.
\end{align*}
The first term equals $O\bigl(\varphi_0(T)(1+\log(\sigmax+1))\bigr)$. For the second, since $b_m > 0$, the geometric series satisfies $\tsum_{j=0}^J 2^{b_m j} = O\big((\sigmax+1)^{b_m}\big)$. Summing over $m \in [M]$ completes the proof.
\end{proof}

\begin{remark}
Every term in the regret bounds of \Cref{thm:OCO-main-result,thm:BCO-main-result} that depends on~$\sigmax$ has $b_m > 0$, so the geometric series absorbs the epoch count for these terms. The $\sigmax$-independent terms (captured by~$\varphi_0$) incur a $1+\log(\sigmax+1)$ overhead because, unlike~$\dtot$ or~$T$, the maximum backlog does not accumulate across epochs: each restart effectively begins a new game, and $\varphi_0$ is charged in full for every epoch. The general convex OCO bound of \Cref{thm:OCO-main-result} does not depend on~$\sigmax$; for it, we use the doubling trick of \citet{bistritz2022} over~$T$ and~$\dtot$ instead.
\end{remark}

\section{Non-Preemptive Framework}\label{app:non-preemptive}

\citet{capacity_constraint2025} also studied a more restrictive variant of capacity-constrained online learning in which preemption is disallowed --- once a round enters the tracking set $\cS$, it must remain until its feedback arrives. Our framework recovers this variant by restricting the support of proxy delay distributions $\gD_t$ to $\{-1, \infty\}$ (Section~\ref{sec:reduction}), so that each round is either rejected outright ($\wtd_t = -1$) or committed to until feedback delivery ($\wtd_t = \infty$). The two-point schedulers in Corollaries~\ref{cor:bernoulli-sigmax} and \ref{cor:pareto-ber-scheduler} are non-preemptive in this sense.

\begin{corollary}\label{cor:pareto-ber-scheduler}
    For each $t \in [T]$, define $H_t = \tsum_{j=1}^t \frac{1}{j}$ and $p_t = \min\cbr{1, \tfrac{C}{16 H_t (d_t+1)}}$. Let the proxy distribution $\gD_t \in \Delta(\{-1, \infty\})$ be given by
    \begin{align*}
        \textstyle \Pr_{\wtd \sim \gD_t}(\wtd = \infty) = p_t,
        \qquad
        \Pr_{\wtd \sim \gD_t}(\wtd = -1) = 1-p_t.
    \end{align*}
    Then, $\Pr(\wtd_t \ge d_t) = p_t$ and $\Pr(|\cS_t| = C) \le \exp(-C)$ for every $t \in [T]$.
\end{corollary}
\begin{proof}
Fix $t \in [T]$. Since $\gD_t \in \Delta(\{-1, \infty\})$ and $d_t \in \Z_{\ge 0}$, we immediately have $\Pr(\wtd_t \ge d_t) = \Pr(\wtd_t = \infty) = p_t$. The rest follows analogously to the proof of Corollary~\ref{cor:pareto-scheduler}.
\end{proof}

Corollary~\ref{cor:pareto-ber-scheduler} requires clairvoyance because its admission probability
depends on $d_t$ at prediction time. The scheduler of Corollary~\ref{cor:bernoulli-sigmax}
requires no advance knowledge of individual delays and can be
implemented using a known upper bound on $\sigmax$.

Since both schedulers satisfy the same saturation and weight moment bounds as their preemptive counterparts, combining them with the analysis of Theorems~\ref{thm:OCO-main-result} and~\ref{thm:BCO-main-result} yields identical regret guarantees in the non-preemptive setting.

\section{Experiments}\label{app:experiments}

{\raggedright
Code for reproducing the experiments is publicly available at \mbox{\url{https://github.com/alex-rbch/cc-oco}}.
\par}

This appendix gives the setup for the capacity ablation in \Cref{sec:experiments}. The domain is the unit Euclidean ball in $\R^5$, the horizon is $T=10^6$, and the tracking capacities are $C\in\{20,40,80,160\}$.

\paragraph{Delays and losses.}
We draw independent nonnegative random variables $Z_t$ with tail
\[
    \Pr(Z_t\ge z)=(1+z/360)^{-1.9},\qquad z\ge0,
    \qquad d_t=\min\{1+\lfloor Z_t\rfloor,T-t\}.
\]
The scale is calibrated to target mean delay $400$ after horizon clipping. For each environment, we round the empirical $0.4$-quantile of the delays upward to the next multiple of $25$ to obtain a threshold $d^*$; this gives $d^*=125$ for every reported seed. We then generate affine losses
\[
    f_t(x)=\tfrac12+\langle a_t,x\rangle,
    \qquad a_t=\tfrac12 g_t\ind\{d_t>d^*\},
    \qquad g_t=0.95g^*+\sqrt{1-0.95^2}\,u_t,
\]
where $g^*=\mathbf{1}/\sqrt{5}$ and $u_t$ is a uniformly sampled unit vector orthogonal to $g^*$. Thus $f_t(x)\in[0,1]$ on the domain, and only longer-delay rounds have nonzero gradients. The entire loss and delay sequence is generated before play. The directional bias $0.95$ was selected on separate pilot seeds $0$--$9$. Final results use seeds $100$--$199$.

\paragraph{Strongly convex losses.}
For the second loss family, we use the same delays and directions and set
\[
    f_t^{\mathrm{SC}}(x)=\frac34 f_t(x)+\frac14\|x\|^2.
\]
Both terms lie in $[0,1]$ on the unit ball, so their convex combination does too. Each loss is $\frac12$-strongly convex, with Hessian $\frac12 I$. The directional signal remains confined to longer-delay rounds, while every round now also contributes the quadratic term.

\paragraph{Tracking and event order.}
The simulator processes positive-delay feedback arriving at time $t$ before prediction at time $t$, with arrivals taking precedence over simultaneous timeout expirations. After prediction, a new positive-delay item is admitted only if a slot is free and its sampled lifetime is positive. Zero-delay feedback is observed immediately after its action without occupying a slot; horizon clipping produces such a delay only at the final round in this experiment. A positive simulator delay $d$ therefore corresponds to delay $d-1$ under the manuscript's end-of-round feedback convention. Across the $100$ seeds, $\sigmax$ has mean $466.67$ and ranges from $457$ to $485$.

\paragraph{Schedulers.}
Bernoulli independently selects each positive-delay round with probability
\[
    p=\min\{1,\,0.75C/\max\{1,\sigmax\}\},
\]
and keeps admitted rounds until feedback arrives. The Pareto baseline samples a nonnegative integer tracking lifetime $L$ with
\[
    \Pr(L\ge d)=(1+d/b_C)^{-1},\qquad d\in\Z_{\ge0},
\]
and clips sampled lifetimes at $T$. We calibrate $b_C$ numerically so that
\[
    \E[\min\{L,T\}]=\sum_{d=1}^T(1+d/b_C)^{-1}=0.75C.
\]
A lifetime of zero rejects a positive-delay item. Both schedulers enforce the hard cap $C$. Retained feedback receives inverse-probability weight $1/p$ or $1+d/b_C$, respectively; dropped feedback contributes zero. These weights correct proxy selection conditional on an available slot. Full tracking retains every round with weight one and no capacity restriction. The proxy distribution and occupancy calibration are empirical adaptations of the schedulers discussed in \Cref{sec:reduction}; the empirical Pareto baseline differs from the time-varying Pareto distribution in Corollary~\ref{cor:pareto-scheduler}.

\paragraph{Learners and tuning.}
The experimental learners use the weighted-gradient feedback of DW-FTRL (\Cref{subsec:DW-OCO}) and the one-point bandit estimator of DW-FTBL (\Cref{subsec:DW-BCO}) for first-order and bandit feedback, respectively. Their updates are implemented as per-arrival projected-gradient steps with fixed parameters, as detailed below.

The OCO learner starts at zero and applies a projected gradient update for each arriving weighted gradient, with constant step size
\[
    \eta_{\mathrm{OCO}}=\frac{2}{\sqrt{T(1+\bar d)}},
    \qquad \bar d=\frac1T\sum_{t=1}^T d_t.
\]
For OCO, feedback is $a_t$ for the affine losses and $\frac34 a_t+\frac12 x_t$ for the strongly convex losses, evaluated at the original played action $x_t$.

For BCO, the learner maintains a center $y_t$ in the ball of radius $1-\delta$ and plays $x_t=y_t+\delta v_t$, where $v_t$ is uniform on the unit sphere. Its estimate is $(5/\delta)f_t(x_t)v_t$, with $f_t$ denoting the chosen loss family, multiplied once by the retention weight. Each arrival triggers a projected update of the center using the raw loss and the full importance-weighted estimate. Define $q_C=1+\bar d/C$ for the constrained methods and $q_C=1$ for full tracking. The fixed parameters within a BCO run are
\[
    \delta=\min\{1,\sqrt{5}\,q_C^{1/4}T^{-1/4}\},
    \qquad
    \eta_{\mathrm{BCO}}=
    \frac{2}{\sqrt{T}\sqrt{\bar d+(5/\delta)^2q_C}}.
\]
All multiplicative tuning constants equal one. The same parameter rules are used for both loss families. The constants remain fixed across capacities, while $\delta$ and $\eta_{\mathrm{BCO}}$ follow the displayed dependence on $C$. Tuning uses the realized mean delay and, for Bernoulli retention, the realized maximum backlog.

\paragraph{Pairing and evaluation.}
The delays, loss directions, and bandit perturbation directions are shared across methods, capacities, and loss families within each seed. The underlying scheduler random streams are also reused across capacities and loss families. Full tracking is run separately for each loss family, and its outcomes are reused unchanged for every $C$. For $A=\sum_{t=1}^T a_t$, the exact static comparator for affine losses is $x^*=-A/\|A\|$ when $A\ne0$ and $x^*=0$ otherwise. For the strongly convex losses, it is $x_{\mathrm{SC}}^*=\Pi(-\frac{3A}{2T})$, where $\Pi$ projects onto the unit ball.

Regret is computed using each learner's actual played actions, including the bandit perturbations. Figure~\ref{fig:capacity-ablation} shows mean regret per round $R_T/T$. Results are averaged across the $100$ seeds, with pointwise $95\%$ Student-$t$ confidence intervals using $99$ degrees of freedom.

\end{document}

%% file: macros.tex
\usepackage{mathabx}

\newcommand{\rbr}[1]{\left(#1\right)}
\newcommand{\sbr}[1]{\left[#1\right]}
\newcommand{\cbr}[1]{\left\{#1\right\}}

\newcommand{\argmin}{\operatornamewithlimits{arg\,min}}

\newcommand{\pair}[1]{\langle{#1}\rangle}

\makeatletter
\newenvironment{algorithm-scheduler}[1][htb]{
    \renewcommand{\ALG@name}{Scheduler}
   \begin{algorithm}[#1]
  }{\end{algorithm}}
\makeatother

\newcommand{\E}{\mathbb{E}}

\newcommand{\ind}{\mathbb I}

\newcommand{\tnorm}[1]{\| #1 \|}
\newcommand{\norm}[1]{\left\| #1 \right\|}

\newcommand{\floor}[1]{\left\lfloor\, {#1}\,\right\rfloor}

\newcommand{\tceil}[1]{\lceil\, {#1}\,\rceil}

\newcommand{\tsum}{\textstyle\sum}

\newcommand\Z{\mathbb Z}
\newcommand\N{\mathbb N}
\newcommand\R{\mathbb R}

\newcommand{\Sphere}{\mathbb{S}^{k-1}}
\newcommand{\Ball}{\mathbb{B}^{k}}

\newlength{\pgmtab}
\def\qedsketch{\ifmmode\Box\else{\unskip\nobreak\hfil
\penalty50\hskip1em\null\nobreak\hfil$\Box$
\parfillskip=0pt\finalhyphendemerits=0\endgraf}\fi}

\newlength{\tpush}
\newcommand{\handout}[5]{
   \noindent
   \begin{center}
   \framebox{ \vbox{ \hbox to \textwidth { {\bf \coursenum\ :\  \coursename} \hfill #5 }
       \vspace{3mm}
       \hbox to \textwidth { {\Large \hfill #2  \hfill} }
       \vspace{1mm}
       \hbox to \textwidth { {\it #3 \hfill #4} }
     }
   }
   \end{center}
   \vspace*{4mm}
   \newcommand{\lecturenum}{#1}
   \addcontentsline{toc}{chapter}{Lecture #1 -- #2}
}

\newcommand{\normop}[1]{{\left\vert\kern-0.25ex\left\vert\kern-0.25ex\left\vert #1
		\right\vert\kern-0.25ex\right\vert\kern-0.25ex\right\vert}}

\newcommand{\cB}{\mathcal{B}}

\newcommand{\cF}{\mathcal{F}}

\newcommand{\cK}{\mathcal{K}}

\newcommand{\cO}{\mathcal{O}}

\newcommand{\cS}{{\mathcal{S}}}

\newcommand{\cW}{\mathcal{W}}

\newcommand{\wtx}{\widetilde{x}}

\newcommand{\wty}{\widetilde{y}}

\newcommand{\htg}{\widehat{g}}

\newcommand{\dtot}{d_{\textnormal{tot}}}

\newcommand{\dmax}{d_{\textnormal{max}}}

\newcommand{\sigmax}{\sigma_{\textnormal{max}}}

\newcommand{\myeq}[1]{\stackrel{\text{(#1)}}{=}}
\newcommand{\myle}[1]{\stackrel{\text{(#1)}}{\le}}

\newcommand{\Unif}{\textnormal{Unif}}

\newcommand{\Pareto}{\textnormal{Pareto}}

\newcommand{\deltatot}{\delta_{\textnormal{tot}}}

\newcommand{\ydelta}{y^{\delta}}

\newcommand{\vdelta}{v^{\delta}}
\newcommand{\fdelta}{f^{\delta}}

%% file: experiments/section.tex
\providecommand{\experimentcodefootnote}{}
We conduct numerical experiments to study how tracking capacity and scheduling affect regret under first-order and bandit feedback. We consider both convex and strongly convex losses, focusing on settings where informative feedback takes longer to arrive. Such dependence can arise when longer or more difficult tasks produce systematically different losses. Retaining this feedback occupies capacity for longer, while favoring short-delay feedback can discard useful information. We model this dependence using heavy-tailed delays, with both the losses and delays fixed before play, consistent with the oblivious-adversary model. See Appendix~\ref{app:experiments} for experimental details.\experimentcodefootnote

For each loss family, we vary the capacity $C$ and compare empirical variants of our uniform-weight Bernoulli scheduler (Corollary~\ref{cor:bernoulli-sigmax}) and the Pareto-type scheduler (Corollary~\ref{cor:pareto-scheduler}), using the same losses and delays across methods and capacities. Full tracking provides an unconstrained baseline.

\begin{figure}[htb]
    \centering
    \providecommand{\capacityfigurewidth}{\linewidth}
    \includegraphics[width=\capacityfigurewidth]{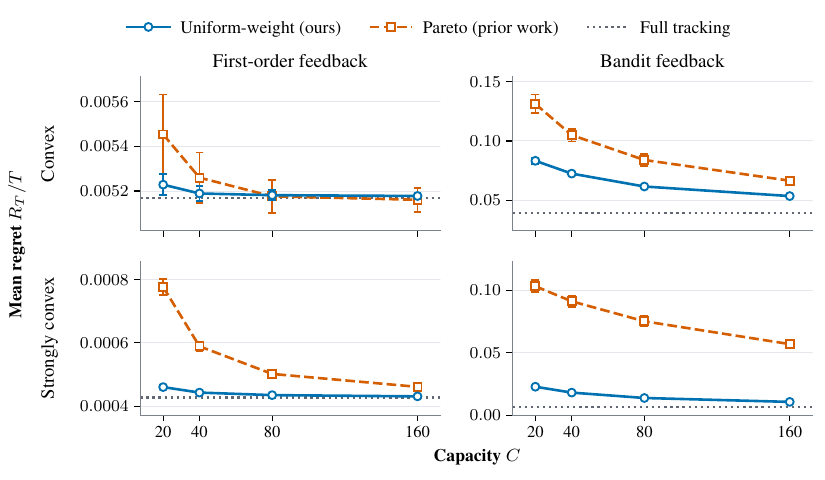}
    \caption{Capacity ablation for convex (top) and strongly convex (bottom) losses under first-order (left) and bandit (right) feedback. Both loss families take values in $[0,1]$. We use $C\in\{20,40,80,160\}$, horizon $T=10^6$, and delays sampled with mean $400$. The plots show mean regret per round $R_T/T$. Markers average $100$ paired repetitions, with error bars indicating $95\%$ confidence intervals. The mean maximum backlog is $\sigmax\approx467$. Horizontal lines indicate full tracking ($C=\infty$).}
    \label{fig:capacity-ablation}
\end{figure}

\Cref{fig:capacity-ablation} shows that scheduling has a greater effect under bandit feedback. In OCO, the uniform-weight scheduler is close to full tracking even at small capacities, while the Pareto-type scheduler approaches this baseline as capacity increases. In BCO, uniform weighting gives consistently lower regret for both loss families. Larger capacities narrow the gap, but both methods remain above full tracking, and the difference between schedulers remains substantial for strongly convex losses. These observations are consistent with our analysis: nonuniform importance weights can amplify the noise in bandit gradient estimates.
\FloatBarrier